\documentclass{article}

\PassOptionsToPackage{numbers, compress}{natbib}

 \usepackage[main, final]{neurips_2026}

\usepackage[utf8]{inputenc} 
\usepackage[T1]{fontenc}    
\usepackage{url}            
\usepackage{booktabs}       
\usepackage{amsfonts}       
\usepackage{nicefrac}       
\usepackage{microtype}      
\usepackage{xcolor}         
\usepackage[most]{tcolorbox}
\usepackage[pagebackref=true,colorlinks,linkcolor=blue,citecolor=brown]{hyperref}
\usepackage{etoc}           
\etocsettagdepth{mainpaper}{none}
\etocsettagdepth{appendix}{subsection}

\usepackage{tabularx}
\usepackage{array}
\newcolumntype{Y}{>{\raggedright\arraybackslash}X}

\newtcbtheorem{mydef}{Definition}{
  colback=blue!5!white,
  colframe=blue!75!black,
  fonttitle=\bfseries,
  arc=2mm,
  boxrule=0.8pt
}{def}

\newtcbtheorem{myremark}{Remark}{
  colback=green!5!white,
  colframe=green!50!black,
  fonttitle=\bfseries,
  arc=2mm,
  boxrule=0.8pt
}{rem}

\usepackage{amsmath,amssymb,amsthm,mathtools}
\usepackage{enumitem}
\usepackage{hyperref}
\usepackage{cleveref}

\newtheorem{proposition}{Proposition}
\newtheorem{theorem}{Theorem}
\newtheorem{lemma}{Lemma}

\usepackage{xcolor}

\usepackage{graphicx} 
\usepackage{wrapfig}
\usepackage{multirow}

\usepackage{algorithm}
\usepackage{algpseudocode}
\usepackage{natbib}

\usepackage{listings}
\usepackage{booktabs}
\usepackage{xcolor}

\lstdefinestyle{prompt}{
basicstyle=\ttfamily\footnotesize,
breaklines=true,
breakatwhitespace=true,
columns=fullflexible,
keepspaces=true,
frame=single,
framerule=0.4pt,
xleftmargin=6pt,
xrightmargin=6pt,
aboveskip=6pt,
belowskip=6pt,
}

\newcommand{\ours}{\textsc{EoM}}

\title{Economy of Minds: Emerging Multi-Agent Intelligence with Economic Interactions}

\newcommand{\harvard}{\raisebox{-0.2\height}{\includegraphics[height=0.8em]{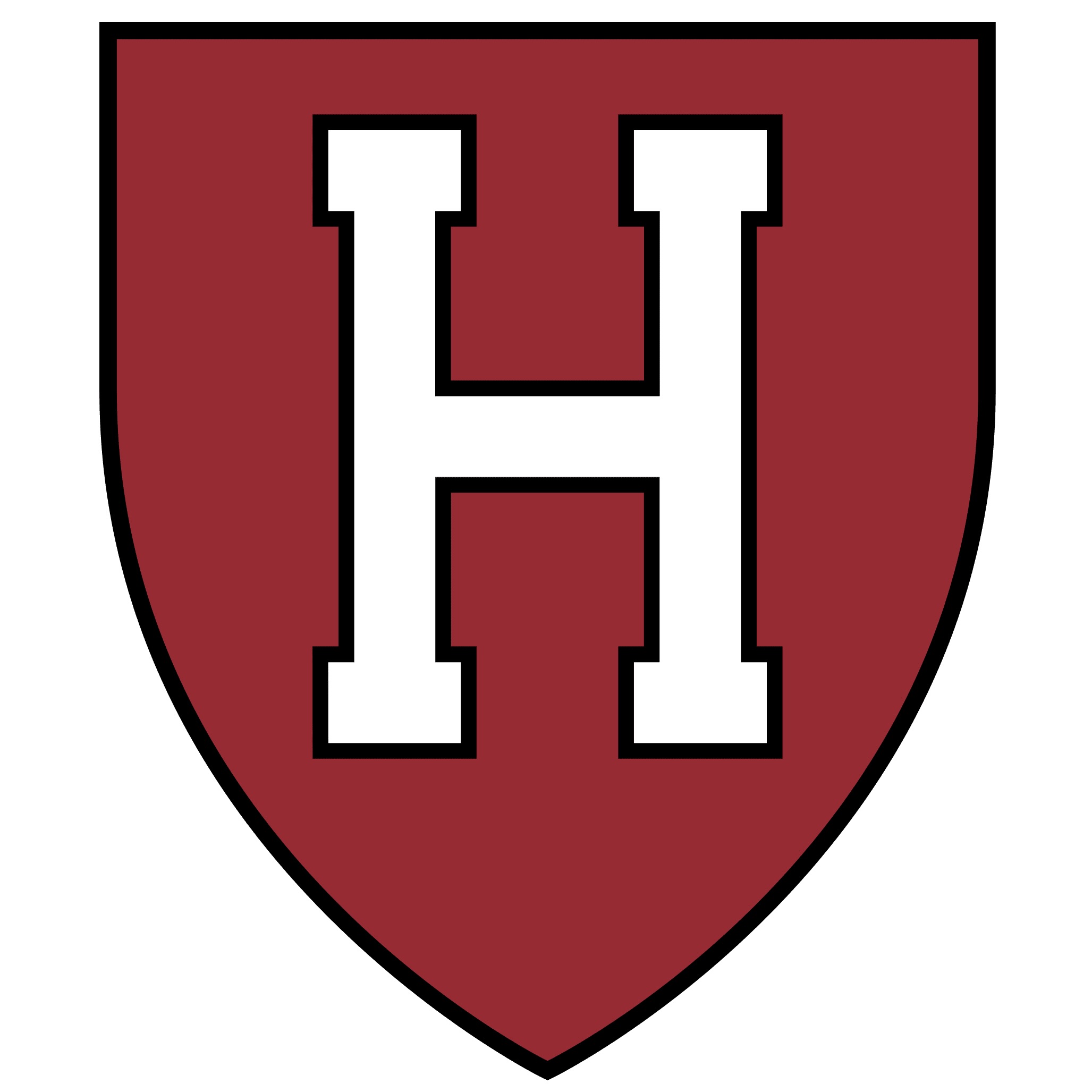}}}

\newcommand{\massit}{\raisebox{-0.2\height}{~\includegraphics[height=0.6em]{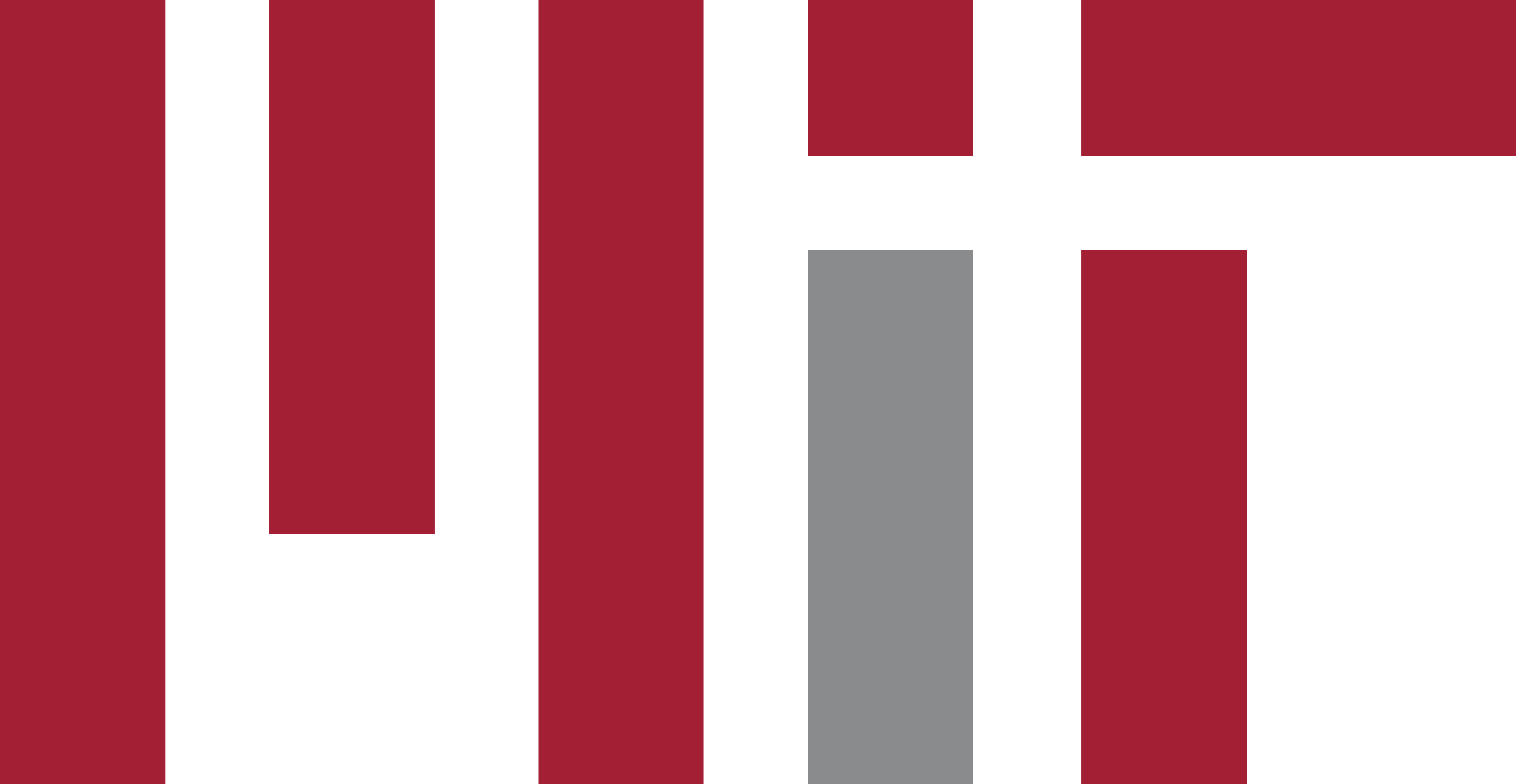}}}

\newcommand{\ai}{\raisebox{-0.2\height}{~\includegraphics[height=0.5em]{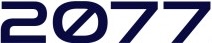}}}

\newcommand{\kempner}{\raisebox{-0.2\height}{\includegraphics[height=0.8em]{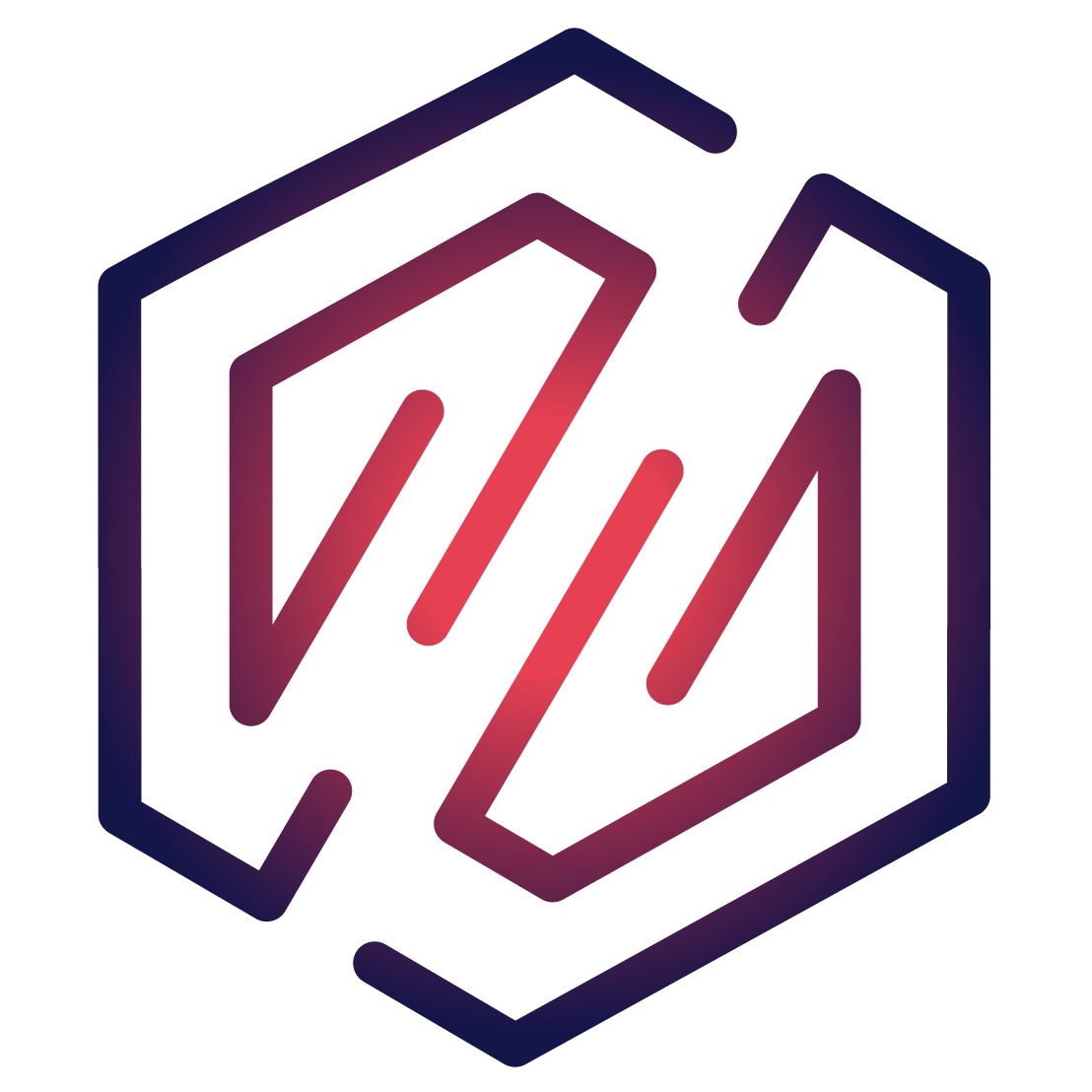}}}

\author{
Zhenting Qi$^{\harvard,*}$ \quad
Huangyuan Su$^{\harvard,\kempner}$ \quad
Ao Qu$^{\massit}$ \quad
Chenyu Wang$^{\harvard}$ \\
\textbf{Yu Yao}$^{\massit}$ \quad
\textbf{Han Zheng}$^{\massit}$ \quad
\textbf{Kushal Chattopadhyay}$^{\harvard}$ \quad
\textbf{Guowei Xu}$^{\harvard}$ \\ 
\textbf{Zihan Wang}$^{\ai}$ \quad
\textbf{Weirui Ye}$^{\massit}$ \quad
\textbf{Vijay Janapa Reddi}$^{\harvard}$ \quad
\textbf{Ju Li}$^{\massit}$ \quad
\textbf{Paul Pu Liang}$^{\massit}$ \\
\textbf{Himabindu Lakkaraju}$^{\harvard}$ \quad
\textbf{Sham Kakade}$^{\harvard,\kempner}$ \quad
\textbf{Yilun Du}$^{\harvard,\kempner,*}$
\vspace{0.7em}
\\
$^{\harvard}$ Harvard \qquad
$^{\massit}$ MIT \qquad
$^{\ai}$ 2077AI \qquad
$^{\kempner}$ Kempner Institute
\vspace{0.7em}
\\
*Corresponding to: \texttt{zhentingqi@g.harvard.edu}, \texttt{ydu@seas.harvard.edu}
\\
\vspace{0.1em}
\\
\href{https://github.com/zhentingqi/EoM}{%
    \raisebox{-0.12em}{\includegraphics[height=1.05em]{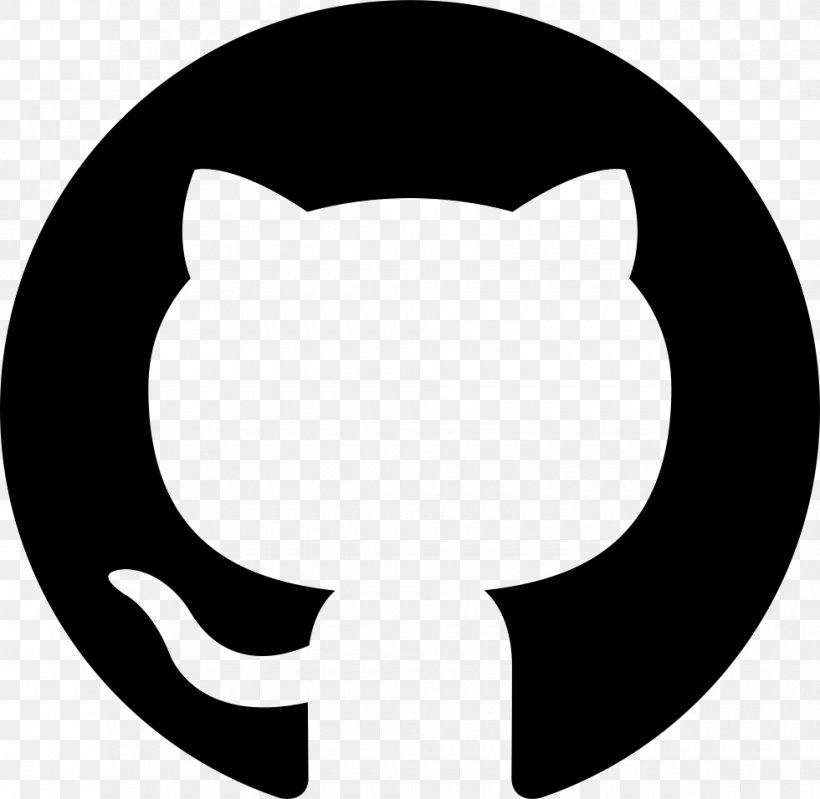}}~\textbf{GitHub}%
}
\qquad
\href{https://zhentingqi.github.io/internal/projects/EoM}{%
    \raisebox{-0.12em}{\includegraphics[height=1.05em]{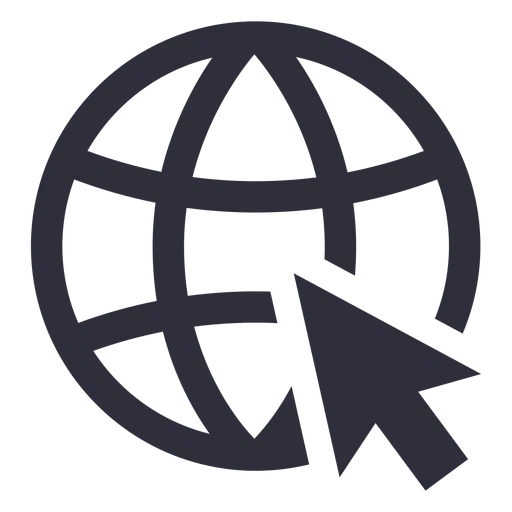}}~\textbf{Project Page}%
}
}

\begin{document}

\maketitle

\etocdepthtag.toc{mainpaper}

\begin{abstract}
How can a population of agents self-orchestrate and self-adapt into stronger collective intelligence without centralized control? Inspired by Friedrich Hayek's economic theory of decentralized coordination in markets, we study this question through an \textit{agent economy} in which agents compete via auctions for the right to act, exchange payments, and accumulate wealth from environmental rewards. These simple local incentives induce decentralized credit assignment, driving global planning without explicit orchestration or communication protocols. The population evolves through economic selection: effective agents accumulate wealth and are mutated via exploitation, while ineffective ones go bankrupt and are replaced via exploration.
Using modern Large Language Model (LLM) based agents, we show that, initialized with weak agents, the agent economy produces emergent multi-step reasoning strategies and outperforms stronger monolithic baselines across five agentic tasks, including mathematical reasoning, financial research, scientific research, accelerator design, and distributed-system optimization. 
Our results suggest a new path to multi-agent intelligence: rather than engineering coordination, we can design decentralized incentive structures under which it automatically emerges.
\end{abstract}

\begin{figure}[h]
    \centering
    \includegraphics[width=\textwidth]{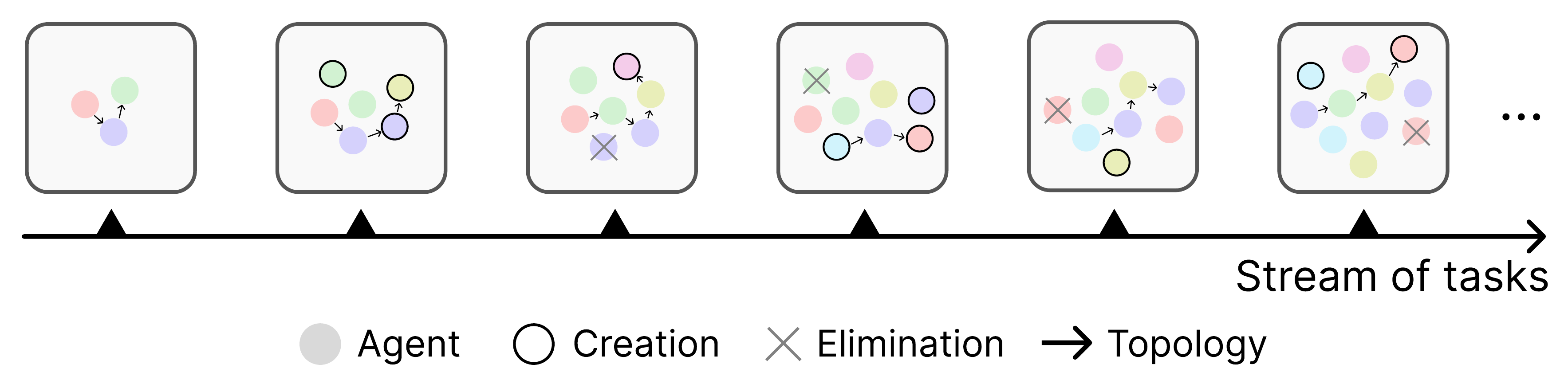}
    \caption{\textbf{Evolution of an agent society over a stream of tasks.} Each panel shows the population at a given stage, where agents are continuously created, selected, connected, and eliminated. As the society encounters more tasks, ineffective agents are removed and corrected, while useful ones persist and diversify, leading to an alive and increasingly structured population.}
    \label{fig.teaser}
\end{figure}

\section{Introduction}

Imagine a world populated by a number of intelligent agents.
Each agent may perform well on certain tasks, but it remains fundamentally limited: each operates with its own priors, partial observations of environments, and bounded computational resources.
When faced with complex tasks that exceed individual capabilities, no single agent can reliably solve problems from start to finish.
How, then, can such a population collectively solve these tasks?

A natural approach is to introduce a central orchestrator that creates agents, assigns specializations, and coordinates actions across agents \citep{hong2023metagpt, qian2024chatdev, wu2024autogen}.
However, such centralized systems suffer from two fundamental limitations.
First, planning is bottlenecked at a single coordination gate: all information and decision-making must flow through the orchestrator, creating both a performance bottleneck and a single point of failure \citep{cemri2025multi, agentnet}.
Second, learning and adaptation become increasingly inefficient as the system scales.
As the number of agents grows, the orchestrator must reason about and manage an ever-expanding set of agents, leading to coordination complexity that grows linearly with system size \citep{agentnet, kim2025towards}.
These limitations motivate a shift in perspective: rather than asking how to build better centralized controllers, we ask whether a group of agents can form decentralized intelligence, which organizes itself into strong collective capability and naturally evolves through time.

These questions are especially compelling today.
Advances in large foundation models such as large language models (LLMs) and world action models make it possible to instantiate intelligent agents that can reason, use tools, and operate in open-ended digital or embodied environments \citep{yao2022react, patil2024gorilla, kim2024openvla, ye2026world}.
However, such agents are typically carefully engineered and tailored to specific domains, introducing inductive biases that might hinder general utility. 
Also, constraints in the underlying models, such as finite context length, partial knowledge, and restricted inference budget, combined with environmental limitations like partial observability and constrained tool access, make it difficult for single agents to solve complex, compositional tasks alone \citep{omidshafiei2017deep, iqbal2022alma, sun2025scaling}. 
Therefore, beyond improving these individual agents, another axis to scale up agentic intelligence is to organize these inherently constrained agents into an effective learnable system as a whole \citep{chen2025multi, xue2025comas, qu2026coral}. 

The economy in human society provides an intuition for such decentralized intelligence. 
In his seminal essay \textit{The Use of Knowledge in Society} \citep{hayek1945knowledge}, Friedrich Hayek argued that the core problem of an economy is not optimization under known information, but the utilization of knowledge that is dispersed across individuals, which cannot be aggregated by any central authority. 
Hayek’s key insight is that the free market solves this problem through the price system: prices serve as signals that aggregate and communicate dispersed information, enabling individuals to coordinate their actions without global awareness. As a result, large-scale social orders emerge from decentralized interactions through competition, specialization, and exchange.
Building on this view, Baum further argued that economic organization provides a concrete model of intelligence, even though individual routines are simple \citep{Baum1996Toward, baum1999toward}. His Hayek machine shows that economic pressure serves as a natural mechanism for credit assignment, determining which routines act, how they are connected, and how they are renewed without centralized evaluation. 


In this work, we show that simple economic incentives are sufficient for modern intelligent agents to self-orchestrate and self-evolve as a society.
We present \textbf{Economy of Minds (\ours{})}, a system where a population of agents compete via auctions for the right to act, exchange payments through peer-to-peer transactions, and accumulate or lose wealth based on an outcome reward.
Agents that consistently contribute to successful trajectories naturally accumulate wealth, thus they survive and are periodically mutated through exploitation, while ineffective agents are eliminated through economic selection and replaced through exploration.
Importantly, each agent operates purely locally: it only requires (1) a wake-up condition, and (2) executes its action accordingly.
There is no need for global awareness, explicit coordination, engineered topology, or prescribed communication protocols.

We implement this framework using LLM-based agents and evaluate it on five digital agentic tasks, where \ours{} improves mathematical reasoning from 15.9\% to 57.0\%, raises financial research performance from 45.0\% to 60.0\%, increases scientific research accuracy over the baseline from 5.0\% to 20.0\%, reduces accelerator-design average energy-delay product to 39.3 versus 80.2 for a strong domain-specific method, and attains a best distributed-system optimization cost of 657 versus 930 for the baseline.
Across these domains, we show with detailed cases that the agent society gradually self-organizes into effective workflows: competent agents persist, mutate, and specialize, while weaker agents are pruned and refined, yielding emergent structure and continual adaptation.

Our findings suggest that economic organization provides a simple, general, and scalable foundation for decentralized multi-agent intelligence.
Rather than explicitly designing coordination mechanisms, we can define an incentive structure under which coordination, specialization, and cooperation naturally emerge between agents.
This points toward a broader paradigm for multi-agent systems in the era of large foundation models: not as centrally engineered pipelines, but as evolving agent societies whose collective intelligence is shaped by the economies they inhabit.

\section{An Economy of Language Agents}
\label{sec:method}

We model a society of language agents interacting through an economic mechanism.
Each agent operates locally, making decisions based only on its own triggering condition and policy, while global coordination emerges from economic interactions.
The system consists of two coupled processes:
(1) \emph{planning}, which governs how agents act and assigns credits to them within an episode, and
(2) \emph{adaptation}, which governs how the population evolves across episodes.

\subsection{Problem Setup}

We consider a task environment modeled as a partially observed Markov decision process
\(
\mathcal{E} = (\mathcal{S}, \mathcal{A}, P, r, \gamma, \mu_0),
\)
where $\mathcal{S}$ is the state space, $\mathcal{A}$ is the action space, $P(s' \mid s,a)$ is the transition kernel, $r:\mathcal{S}\times\mathcal{A}\to\mathbb{R}$ is the reward function, $\gamma \in (0,1]$ is the discount factor, and $\mu_0$ is the initial-state distribution.
At step $t$, the system observes $o_t \in \mathcal{O}$; in fully observed settings, $o_t=s_t$.

Each agent is implemented by a language model parametrized by $\theta$. Here we adopt a simplified setting, where we use a shared frozen model backbone for all agents, with diversity arising entirely through system prompts.
Formally, an agent is a tuple
\(
a = (\phi_a,\pi_a,b_a,W_a),
\)
where $\phi_a:\mathcal{O}\to\{0,1\}$ is a \emph{triggering predicate} that determines whether the agent is eligible to act, $\pi_a:\mathcal{O}\to\Delta(\mathcal{A})$ is its action policy, $b_a \in \mathbb{R}_{\ge 0}$ is a fixed bid associated with the agent, and $W_a \in \mathbb{R}$ is its current wealth.
Both $\phi_a$ and $\pi_a$ are instantiated by the same frozen LLM with agent-specific prompts
$p_a=(p_a^{\mathrm{trig}}, p_a^{\mathrm{act}})$:
\begin{align*}
\phi_a(o) = \mathrm{LLM}_\theta(p_a^{\mathrm{trig}}; o) \in \{0,1\}, \quad \pi_a(\cdot \mid o) = \mathrm{LLM}_\theta(p_a^{\mathrm{act}}; o) \in \Delta(\mathcal{A}).
\end{align*}
Thus, an agent is fully specified by its prompt pair together with its fixed bid and current wealth.\footnote{This formulation generalizes Baum's Hayek machine \citep{baum1999toward, baum2000evolution} from hand-specified condition-action rules to general agents: the condition is a model-based predicate, and the action is drawn from a model-based policy rather than being a fixed move.}
At episode $e$, the active population is denoted by $\mathcal{P}_e$.








\subsection{Planning with Auctions and Transactions}
\label{sec.method.planning}

\begin{figure}[h]
    \centering
    \includegraphics[width=\textwidth]{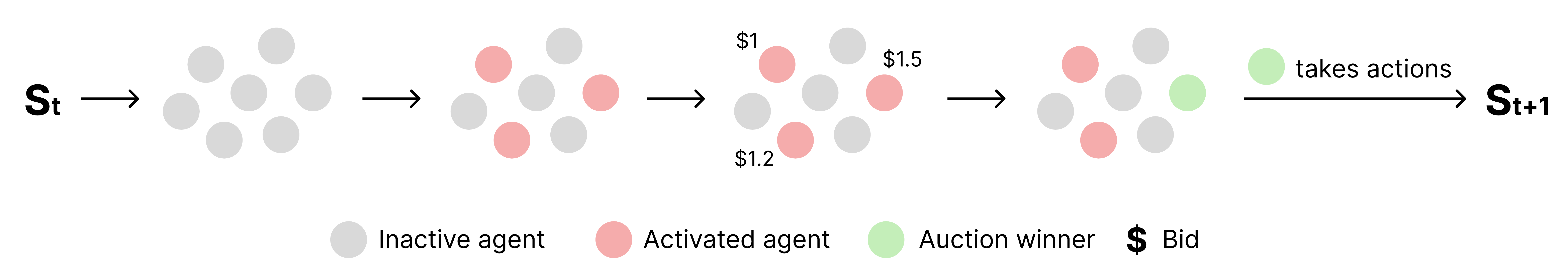}
    \caption{\textbf{Auctions.} Agents whose wake-up conditions are satisfied become eligible to bid; the highest bidder wins the auction, executes the action, and advances the environment from $s_t$ to $s_{t+1}$.}
    \label{fig.auction}
\end{figure}

\begin{figure}[h]
    \centering
    \includegraphics[width=0.85\textwidth]{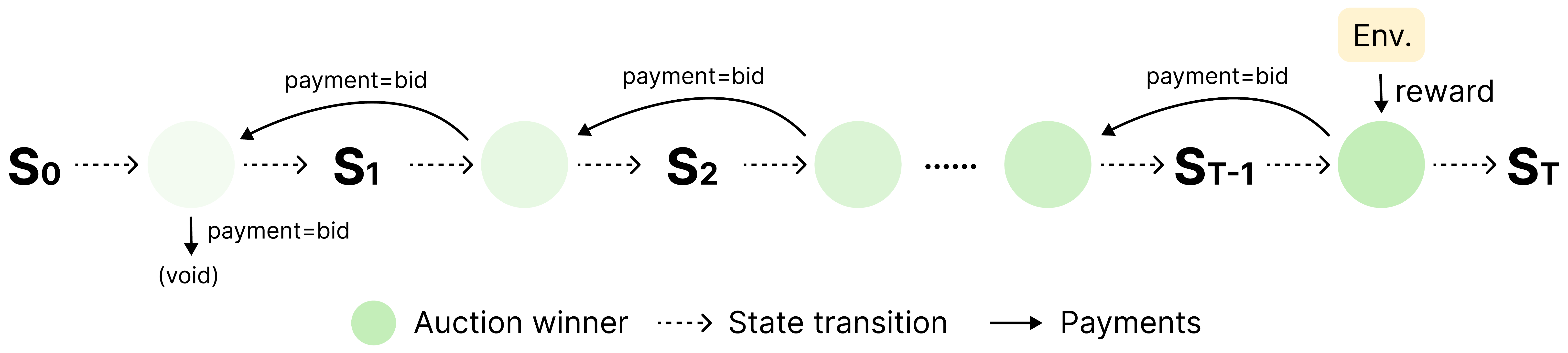}
    \caption{\textbf{Transactions.} Credit assignment naturally emerges as profits flow backward through the action sequence, rewarding agents whose actions enable successful downstream outcomes.}
    \label{fig.transaction}
\end{figure}





At each environment step, agents compete for control through an auction (\Cref{fig.auction}).
Given the current observation $o_t$, each agent evaluates its triggering predicate to determine eligibility.
The eligible set is
\(
E_t = \{a \in \mathcal{P}_e : \phi_a(o_t)=1\}.
\)
If $E_t=\emptyset$, no agent acts and the episode terminates or defaults to a null transition, depending on the environment.
Otherwise, the auction winner is the highest-bidding eligible agent,
\(
a_t^\star \in \arg\max_{a\in E_t} b_a,
\)
with ties broken randomly.

The auction serves as a decentralized action-selection mechanism.
Rather than relying on a centralized policy, control is allocated to the agent that bids the highest value to acting in the current context, as reflected by its bid.
Once selected, the winning agent $a_t^\star$ samples an action from its policy
\(
a_t^{\mathrm{env}} \sim \pi_{a_t^\star}(\cdot \mid o_t),
\)
which advances the environment and produces the next observation $o_{t+1}$ and reward $r_t$.
Let $a_{t-1}^\star$ denote the previous winning agent in the same episode.
We then apply a bucket-brigade transfer rule (\Cref{fig.transaction}):
\begin{align}
W_{a_t^\star} &\leftarrow W_{a_t^\star} - b_{a_t^\star} + r_t, \quad
W_{a_{t-1}^\star} \leftarrow W_{a_{t-1}^\star} + b_{a_t^\star}.
\label{eq:payments}
\end{align}
That is, the winner pays its bid to the previously active agent while collecting any environmental reward $r_t$.
For the first winner in an episode, the payment is made to the house.

This payment rule yields a decentralized form of credit assignment.
An agent profits not only by directly receiving reward, but also by taking actions that place the system in states from which downstream agents are willing to pay highly for control.
As a result, value flows backward along successful trajectories: agents whose actions enable productive continuations accumulate wealth, while agents that lead the system into unproductive states lose it.

\subsection{Adaptation with Exploration and Exploitation}
\label{sec.method.adaptation}





Beyond within-episode planning, the population evolves across episodes through economic selection.
Let $\mathcal{G}$ denote a prompt-generation operator that is attached to the agents and proposes new agent prompts from the current one, for example, by mutating a successful agent's or amending a failed one's system prompt.
Each newly created agent is initialized with wealth $W_0 \geq 0$, while existing agents may also incur a periodic rent $\rho \ge 0$. Examples of adaptations are provided in Appendix~\ref{sec:case-study-scientific-research} and \ref{sec:case-study-cloudcast}, and theoretical motivations are stated in Appendix~\ref{sec:theory}.

\paragraph{Exploitation.}
Agents that consistently contribute to successful trajectories automatically accumulate wealth and persist in the population.
Wealthy agents are periodically selected as parents and are mutated to produce new agents, allowing their successful patterns to be reused and refined.
This mechanism biases the population toward high-performing behaviors, reinforces effective strategies, and promotes the emergence of specialized roles.
Concretely, wealthy agents use their prompt generators $\mathcal{G}$ to propose prompt mutations that preserve their useful wake-up conditions or action policies while introducing small behavioral variations.

\paragraph{Exploration.}
Agents lose wealth through unhelpful actions or prolonged inactivity.
When their wealth becomes negative, they are removed from the population, and new agents are introduced to replace them, typically by random or complementary variation of these bankrupt agents.
This continual turnover enables learning from failures, discoveries of new behaviors, and prevents premature convergence.
In this case, the prompt generator $\mathcal{G}$ is used by bankrupt agents to propose amended prompts that correct their failure modes or explore complementary regions of the behavior space.

Formally, between episodes the population update consists of three stages:
\begin{enumerate}[leftmargin=*]
    \item \textbf{Rent:} each agent pays $\rho$, so $W_a \leftarrow W_a - \rho$;
    \item \textbf{Removal:} agents with $W_a < 0$ are deleted;
    \item \textbf{Injection:} new agents are added according to the exploitation and exploration mechanisms until the population satisfies the prescribed maximum size constraints.
\end{enumerate}
This induces a stochastic population process over finite sets of prompted agents, driven jointly by environment randomness, LLM stochasticity, and prompt generation.
 
A key design choice is that bids are not learned online; instead, each agent receives a bid when it is introduced, which is then frozen.
For a newly injected agent $a'$ (novice), let $t$ denote its first step at which it becomes eligible, and let
\(
C_t = \{a \in \mathcal{P}_e \setminus \{a'\} : \phi_a(o_t)=1\}
\)
be the set of competing eligible agents at that step.
We assign its bid according to the novice rule
\begin{equation}
b_{a'} = \Bigl(\max_{a\in C_t} b_a\Bigr) + \varepsilon_{a'},
\qquad \varepsilon_{a'} \sim \mathcal{D}_\varepsilon,
\label{eq:novice_rule}
\end{equation}
with $\max \emptyset := 0$.
In our experiments, $\mathcal{D}_\varepsilon$ is a small positive perturbation distribution.
This rule guarantees that the new agent wins the first auction for which it is eligible, forcing the system to test it at least once before market selection determines whether it should survive (more discussion in Appendix~\ref{sec:pseudo_code}).
Together, exploitation and exploration drive a self-evolving system.
Exploitation preserves and sharpens useful behaviors, while exploration injects novelty and enables adaptation.
Crucially, evolution is governed entirely by economic signals, i.e., wealth accumulation and loss, without centralized supervision or explicit global performance labeling.






\subsection{Training and Evaluation}

During optimization (\Cref{alg:training}), each task episode uses the auction-based planning mechanism from \Cref{sec.method.planning}: agents bid for control, winning agents act in the environment, bid transfers propagate wealth along the trajectory, and environment rewards update the final actor. These interactions provide an implicit evaluation signal, after which the adaptation step in \Cref{sec.method.adaptation} removes bankrupt agents, retains profitable ones, mutates successful templates, and replenishes the population to preserve diversity. For evaluation (\Cref{alg:evaluation}), we reuse the same planning structure but disable the evolutionary and economic updates. The trained population is frozen, bids are fixed, no payments, rewards, rent, births, or mutations are applied, and each test task runs on a thread-local snapshot to avoid shared-state effects. Thus, optimization uses auctions for both action selection and credit redistribution, while evaluation uses them only as a decentralized decision rule, measuring the behavior of the learned population without further learning or wealth dynamics.

\section{Experiments}
\label{sec.exp}

\subsection{Setup}

\textbf{Partial vs. complete agents.}
We use \emph{partial agent} to denote any agent whose capability is intentionally incomplete to solve the full task. This partiality can take different forms across domains: an agent may have a restricted action space, access to only one tool, a short generation budget, specialized roles, or only partial observation of the environment. 
A \emph{complete agent}, by contrast, has access to the full task interface, thus can use the full action space and attempt to solve the task end-to-end. This distinction lets us test whether economic organization can compensate for, or even outperform, capability concentrated in a single complete agent.

\textbf{Tasks.}
We instantiate \ours{} on five domains. On \textbf{mathematical reasoning}, we use {MATH}~\citep{hendrycks2021measuring} to train on an easy-to-hard task stream from Level 1 to 5 and evaluate greedy pass@1 accuracy by difficulty level. The population is initialized with planner, executor, and verifier agents, each with short output budgets of 128 tokens on average. On \textbf{financial research}, we use {Finance-Agent-Bench}~\citep{bigeard2025finance}, where agents solve financial questions over company filings using four tools; each partial agent has access to only one tool. For \textbf{scientific research}, we use {FrontierScience-Research}~\citep{wang2026frontierscience}, where agents are tasked to solve open-ended scientific questions with literature, planner, executor, and verifier roles. For \textbf{accelerator design}, we use the \textsc{Gemmini} benchmark suite~\citep{genc2021gemmini} and optimize mappings for 24 ResNet-50 convolution kernels, with minimizing energy-delay product (EDP) as the objective. The population is initialized with three role-specialized partial agents: \emph{Historians} that summarize prior trial outcomes, \emph{Planners} that propose mapping-level search directions, and \emph{Executors} that carry out local mapping evaluations. For \textbf{distributed-system optimization}, we use {Cloudcast}, a task from ADRS~\citep{cheng2025barbarians}, in which agents iteratively improve a program to minimize the total data-transfer cost. Full task protocols, splits, agent definitions, tools, and model backbones are provided in Appendix~\ref{sec:exp-details}.

\textbf{Baselines.}
Complete-agent baselines include \textsc{ReAct}~\citep{yao2022react}, which solves each task end-to-end with the same backbone and full task interface; \textsc{GEA}~\citep{weng2026group}, which performs self-improvement through experience sharing and agent evolution; and \textsc{OpenEvolve}~\citep{openevolve,novikov2025alphaevolve}, an evolutionary monolithic coding-agent baseline. As a partial-agent baseline, we use Multi-Agent Debate~\citep{du2024improving}, where several partial agents interact by communication but do not use market-driven population evolution. For domain-specific comparisons, we also include \textsc{DOSA}~\citep{hong2023dosa} for accelerator design. Additional baseline details are in Appendix~\ref{sec:exp-baselines}.

\subsection{Can Economics Turn Weak Individuals into Stronger Systems?}

Across all five domains, we show that economic coordination turns individually partial agents into collective systems that match or outperform complete-agent baselines. On MATH, \ours{} improves Llama-3.1-8B agents from 15.9\% to 57.0\% and Gemma-2-9B agents from 4.2\% to 45.1\%, exceeding the corresponding complete-agent baselines of 51.9\% and 44.3\% in \Cref{tab:math-and-accelerator} (left). This is notable because each individual agent in the population is role-specialized and restricted to short outputs, while the complete-agent baseline can attempt the full problem end-to-end without restrictions.
On accelerator design, \ours{} reduces average EDP to 39.3, compared with 43.1 for the same-backbone complete \textsc{ReAct} agent and 80.2 for the domain-specific \textsc{DOSA} baseline in \Cref{tab:math-and-accelerator} (right). 

\Cref{fig:three-panels} shows that, on Finance-Agent-Bench, \ours{} rises from 45.0\% at initialization to 60.0\% after 30 training tasks. This outperforms Multi-Agent Debate at 50.0\%, \textsc{ReAct} at 45.0\%, and \textsc{GEA} at 50.0\%, even though each partial agent in \ours{} can access only one tool. 
On FrontierScience-Research, \ours{} reaches 8.5\% mean accuracy and 20.0\% best-run accuracy, compared with 1.8\% mean and 5.0\% best-run accuracy for \textsc{GEA} under the same Gemini-3-Flash backbone. 
Finally, on Cloudcast, \ours{} reaches an average total cost of 673 across three attempts, with the best attempt achieving 657, compared with 930 for \textsc{OpenEvolve}, corresponding to a 28\% reduction in best cost while using fewer optimization episodes.

Together, these results show that the relevant advantage is not merely ``many agents'' over ``one agent.'' Rather, \ours{} shows that a population of partial agents, when organized by economic interactions, can match or surpass complete agents with greater individual access to the task interface.

\begin{figure}[ht]
\centering

\begin{minipage}{0.25\linewidth}
    \centering
    \includegraphics[width=\linewidth]{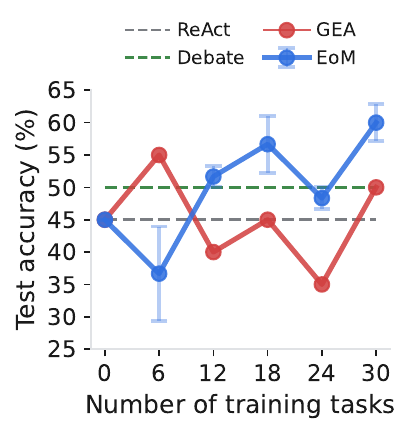}
    
    \footnotesize{(a) Finance-Agent-Bench}
\end{minipage}
\hfill
\begin{minipage}{0.29\linewidth}
    \centering
    \includegraphics[width=\linewidth]{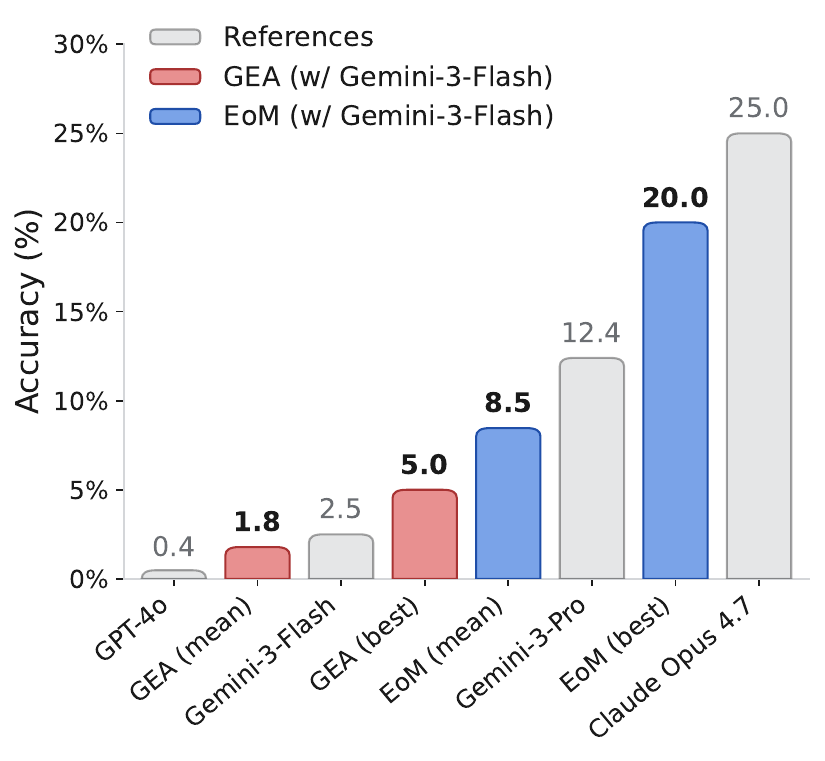}
    
    \footnotesize{(b) FrontierScience}
\end{minipage}
\hfill
\begin{minipage}{0.425\linewidth}
    \centering
    \includegraphics[width=\linewidth]{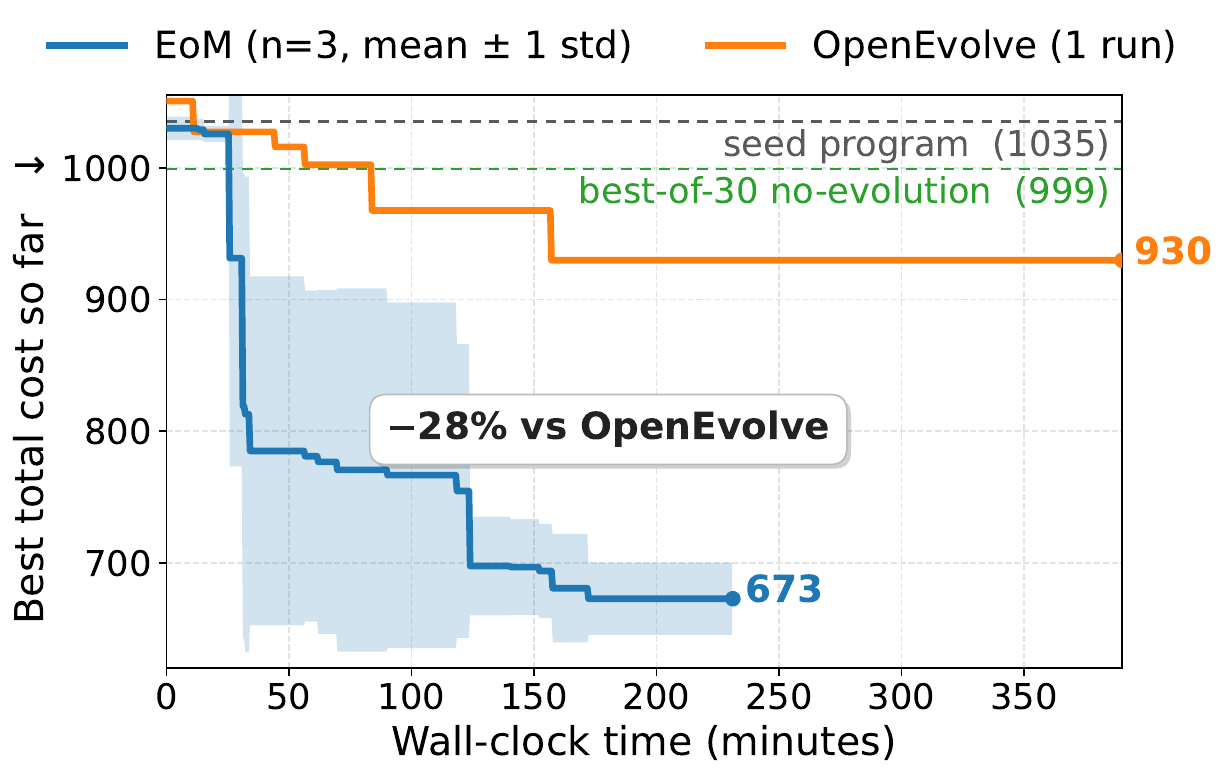}
    
    \footnotesize{(c) Cloudcast}
\end{minipage}

\caption{\small \textbf{Performance across domains.}
\ours{} consistently outperforms baselines, demonstrating the benefits of economic coordination among agents.}
\label{fig:three-panels}

\end{figure}

\begin{table}[t]
\centering
\scriptsize
\caption{\small \textbf{Performance across domains.} 
Left: MATH accuracy comparing constrained populations to complete agents (* denotes officially reported numbers). 
Right: Accelerator design results measured by average EDP (lower is better). 
In both settings, \ours{} achieves stronger performance than corresponding baselines.}
\label{tab:math-and-accelerator}

\begin{minipage}{0.47\linewidth}
\centering
\setlength{\tabcolsep}{6pt}
\begin{tabular}{lccc}
\toprule[1.5pt]
\multirow{2}{*}{\textbf{Backend}} & \multicolumn{2}{c}{\textbf{Partial agents}} & \multirow{2}{*}{\textbf{Complete agent}} \\
\cmidrule(lr){2-3}
 & Initial & After training & \\
\midrule
Llama-3.1-8B & 15.9 (1.37) & \textbf{57.0} (3.36) & 51.9* \\
Gemma-2-9B   & 4.2 (0.52)  & \textbf{45.1} (4.12) & 44.3* \\
\bottomrule[1.5pt]
\end{tabular}

\end{minipage}
\hfill
\begin{minipage}{0.47\linewidth}
\centering
\setlength{\tabcolsep}{6pt}
\begin{tabular}{lc}
\toprule[1.5pt]
\textbf{Method} & \textbf{Avg.\ EDP ($\mu$J$\cdot$Mcyc) $\downarrow$} \\
\midrule
DOSA~\cite{hong2023dosa}      & 80.2 \\
Complete agent (Gemma-4-31B-it) & 43.1 \\
\ours{} (Gemma-4-31B-it)         & \textbf{39.3} \\
\bottomrule[1.5pt]
\end{tabular}

\end{minipage}

\end{table}

\subsection{Beyond Multiple Agents: The Role of Economic Ingredients}

The gains are not explained by merely having multiple agents: they depend on the economic dynamics that allocate control, transfer value, remove unproductive agents, and propagate successful ones. The ablations in \Cref{tab:ablations} show that weakening these dynamics consistently reduces performance.

On MATH, the original system achieves the strongest partial-agent performance, with 43.9 mean accuracy and 57.0 best-run accuracy. Perturbing the economic parameters lowers performance: increasing rent, decreasing rewards, or increasing rewards reduces the mean to 39.0--41.8 and the best run to 44.0--47.0. This indicates that performance depends on the balance between reward inflow, rent pressure, and agent survival.

On Finance-Agent-Bench, the full system again achieves the strongest overall result, with 52.5 mean accuracy and 65.0 best-run accuracy. Removing exploration causes a large drop to 26.0 mean and 40.0 best accuracy, while removing exploitation lowers the mean to 33.5. Removing auctions also underperforms the full system, reaching 48.0 mean and 58.5 best accuracy. These results suggest that the population needs both sides of the economic process: exploration introduces new candidate agents with awareness of past failures, while exploitation propagates successful ones.

Cloudcast provides a complementary comparison. In \Cref{fig:three-panels} (c), \ours{} reaches a best cost of 673, while the best-of-$N$ multi-agent baseline reaches only 999. Since this baseline uses multiple agents but does not evolve them through market selection, the gap shows that repeated multi-agent sampling alone is insufficient. Thus, economic dynamics are not incidental implementation details; they are the mechanism that turns a collection of partial agents into an adaptive society.

\begin{table}[h]
\centering
\scriptsize
\caption{\textbf{Ablations on MATH (left) and Finance-Agent-Bench (right).} 
Sensitivity of performance to economic parameters (left) and component removal (right). 
In both cases, the full/original system achieves the strongest overall results.}
\label{tab:ablations}

\begin{minipage}{0.49\linewidth}
\centering
\setlength{\tabcolsep}{5pt}
\begin{tabular}{llcc}
\toprule[1.5pt]
\multicolumn{2}{c}{\textbf{Agent Configuration}} & \textbf{Mean (\%)} & \textbf{Best (\%)} \\ \midrule
Complete & / & 51.9 & 51.9 \\ \midrule
\multirow{4}{*}{Constrained}
& large rent ($\times10$) & 41.8 & 47.0 \\
& small reward ($\times0.2$) & 39.0 & 44.0 \\
& large reward ($\times4$) & 40.9 & 47.0 \\
& original & \textbf{43.9} & \textbf{57.0} \\ \bottomrule[1.5pt]
\end{tabular}

\end{minipage}
\hfill
\begin{minipage}{0.49\linewidth}
\centering
\setlength{\tabcolsep}{5pt}
\begin{tabular}{llcc}
\toprule[1.5pt]
\multicolumn{2}{c}{\textbf{Agent Configuration}} & \textbf{Mean (\%)} & \textbf{Best (\%)} \\ \midrule
Complete & / & 45.0 & 45.0 \\ \midrule
\multirow{4}{*}{Constrained}
& w/o auction & 48.0 & 58.5 \\
& w/o exploration & 26.0 & 40.0 \\
& w/o exploitation & 33.5 & 60.0 \\
& full & \textbf{52.5} & \textbf{65.0} \\ \bottomrule[1.5pt]
\end{tabular}

\end{minipage}
\end{table}

\subsection{How Does the Economy Improve Performance?}

We show that the economy improves performance by selecting useful action chains, reproducing successful lineages, eliminating unproductive agents, and gradually sharpening both agent prompts and interaction topology.

\textbf{What changes inside the society as performance improves?}
The training dynamics reveal that \ours{} improves by reshaping both the population and the agents themselves. On Finance-Agent-Bench, \Cref{fig:three-panels} (a) shows a non-monotonic but improving trajectory: \ours{} starts at 45.0, dips during early exploration, then recovers and finishes at 60.0. This pattern is consistent with a market that first reallocates control and tests alternative specialists before converging to stronger coordination.

Accelerator design gives a more direct view of the economic mechanism. In \Cref{fig:dse_dynamics}, per-agent wealth trajectories show how useful lineages survive while ineffective ones disappear. In panel~(a), Historian children born during the run quickly lose wealth and are removed, indicating that the inherited mutation could not earn back its rent. In panel~(b), a Planner lineage spawns two good-birth offspring and continues to dominate the auction, while a Historian-line bad-birth child eventually bankrupts. In panel~(c), a successful Historian lineage and a failing Executor lineage evolve in parallel. Across these cases, wealth concentrates on agents that repeatedly contribute to record-breaking submissions, while low-value agents are pushed toward bankruptcy. In short, the economy improves performance by turning trajectory-level success into population-level selection.

\begin{figure}[t]
  \centering
  {\IfFileExists{figs/dse_dynamics.pdf}{\includegraphics[width=\linewidth]{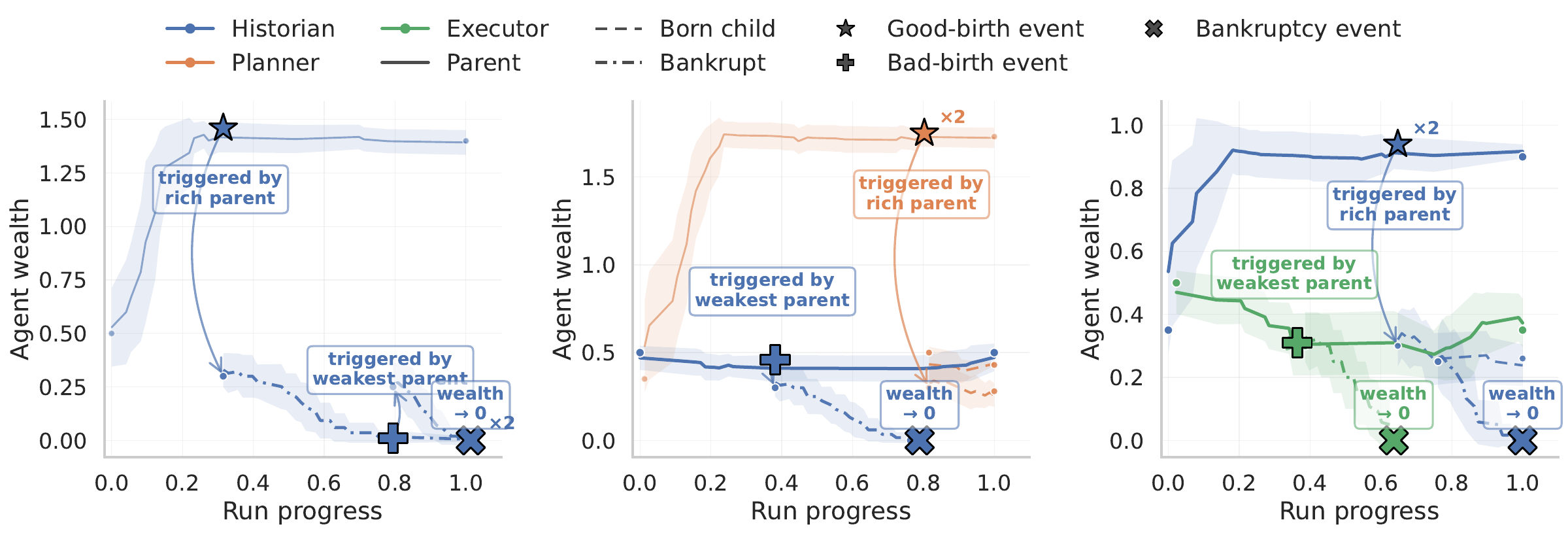}}{\fbox{\parbox[c][0.22\textheight][c]{0.95\linewidth}{\centering Placeholder for accelerator-design training dynamics figure.}}}}
  \vspace{-4mm}
  \caption{\small\textbf{Training dynamics in accelerator design.} Per-agent wealth on three representative ResNet-50 kernels. Wealth flow to agents that produce new EDP records; rent uniformly deducts wealth. Periodic births spawn \emph{good-birth} children ($\star$, \emph{exploitation}: mutated from the richest agent) and \emph{bad-birth} children ($+$, \emph{exploration}: amended from the weakest); wealth $<0$ triggers bankruptcy ($\times$). Shaded bands are rolling $\pm 1\sigma$. \textbf{(a)} Both Historian descendants bankrupt---inherited bias fails market pressure. \textbf{(b)} A Planner lineage reproduces twice while a Historian bad-birth child eventually fails. \textbf{(c)} A strong Historian and a struggling Executor lineage co-exist.}
  \label{fig:dse_dynamics}
\end{figure}

\textbf{Does the society learn reusable structure?}
We show that the society learns more than task-specific answer traces. \Cref{fig:mechanism-robustness} (a) shows that \ours{} achieves a $2.2\times$ geometric-mean EDP gain over \textsc{DOSA} across all 24 ResNet-50 convolution kernels, with much larger gains on the hardest kernels: $37.5\times$, $26.3\times$, $17.3\times$, and $12.0\times$ on \texttt{Conv} 14, 16, 17, and 4. These gains are structured rather than uniform. The hardest kernels are the $1{\times}1$ convolutions inside ResNet-50's bottleneck blocks, which have large input- and output-channel counts but small spatial dimensions. For such shapes, an \emph{output-stationary} dataflow, which holds each output partial sum in fast on-chip storage and accumulates contributions along the input-channel dimension, is a known effective design pattern.
Importantly, \ours{} is not given this motif as a template. The auction rewards only EDP record-breaks, with no per-kernel labels, dataflow-specific reward shaping, or hand-coded output-stationary preference. Nevertheless, across the strongest solutions, the population repeatedly converges on the same tiling pattern, recovering a transferable hardware--software co-design heuristic that DOSA misses. Thus, market selection can discover reusable domain structure when successful agents are allowed to accumulate wealth and propagate through mutation.

Beyond this domain-specific behavior, we observe two broader forms of adaptation. The population evolves sharper prompt-level strategies, as illustrated in Appendix~\ref{sec:case-study-prompt-evolution}, enables cross-domain transfer, as shown in Appendix~\ref{sec.case-study-cross-domain} and reorganizes its interaction topology, as shown in Appendix~\ref{sec:topology-evolution-research}. Appendix~\ref{sec:case-study-cloudcast} traces both adaptations in a Cloudcast run. Together, these analyses suggest that economic dynamics shape both internal agent policies and macroscopic social structure.

\subsection{Robustness and Generalization}

We show that the evolving society is robust in three complementary senses: skills learned on easier tasks transfer to harder ones, performance depends on curriculum but does not collapse under reversed ordering, and adding a complete generalist does not destroy specialization.

\textbf{Do learned behaviors transfer from easier tasks to harder ones?}
On MATH, training on an easy-to-hard stream improves not only the easier levels encountered earlier in training, but also harder levels that are initially beyond the agents' capabilities. As shown in \Cref{fig:math-levels}, both Llama-3.1-8B and Gemma-2-9B improve on every difficulty band. The largest gains appear on Levels~1--3, where Llama-3.1-8B rises to roughly 55--70\% and Gemma-2-9B to roughly 45--65\%. Importantly, improvement also transfers to Level~5: performance rises from around 10\% at the beginning to about 20\% by the end for both backbones. Thus, local reasoning routines learned on simpler problems can be recomposed on harder ones.

\begin{figure}[ht]
    \centering
    \includegraphics[width=\textwidth]{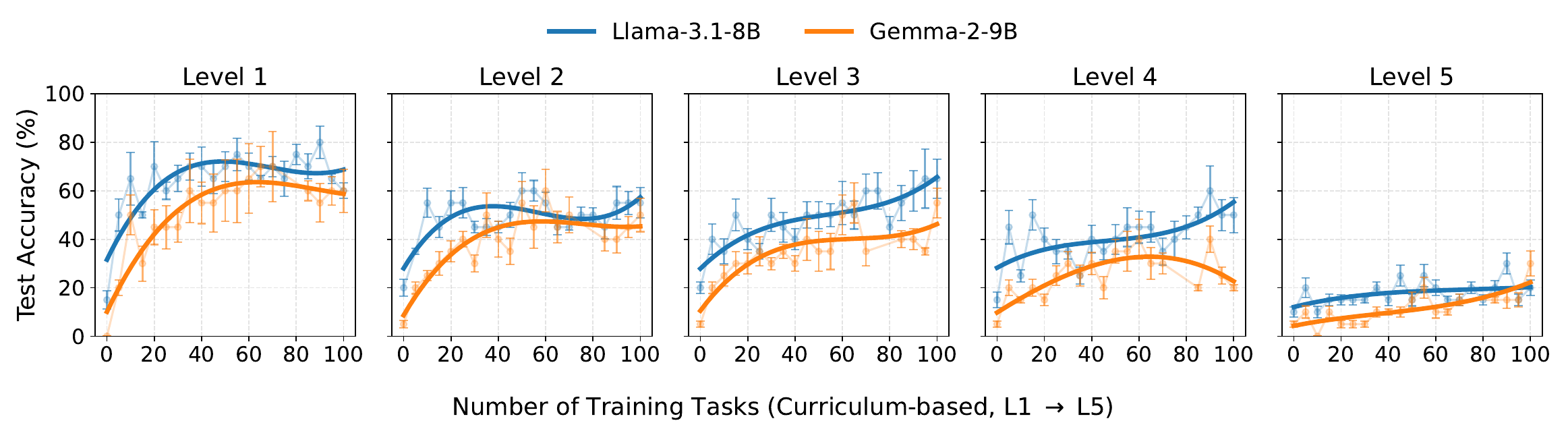}
    \vspace{-4mm}
    \caption{\small \textbf{Easy-to-hard generalization on MATH.} Test accuracy across MATH difficulty levels during training. The partial agent population improves not only on the easier levels encountered earlier, but also on harder levels that are initially beyond its capability, indicating that behaviors learned on simple problems can be reused on more difficult ones.}
    \label{fig:math-levels}
\end{figure}

\textbf{How sensitive is the society to curriculum order?}
We compare the default easy-to-hard curriculum with a reversed hard-to-easy schedule. 
\Cref{fig:mechanism-robustness} (b) shows that both schedules improve quickly at the beginning, but the easy-to-hard curriculum stays ahead for most of training and finishes clearly higher, at roughly 57\% versus about 47\% for the reversed schedule. The reversed curriculum plateaus in the low 40s for much of training, while the easy-to-hard curriculum continues improving late in training. This indicates that partial specialists benefit from first mastering reusable local routines before confronting the hardest problems, although the system still improves under the reversed order.

\begin{figure}[ht]
    \centering

    \begin{minipage}{0.48\linewidth}
        \centering
        \IfFileExists{figs/dse_perlayer_ratio.pdf}
        {\includegraphics[width=\linewidth]{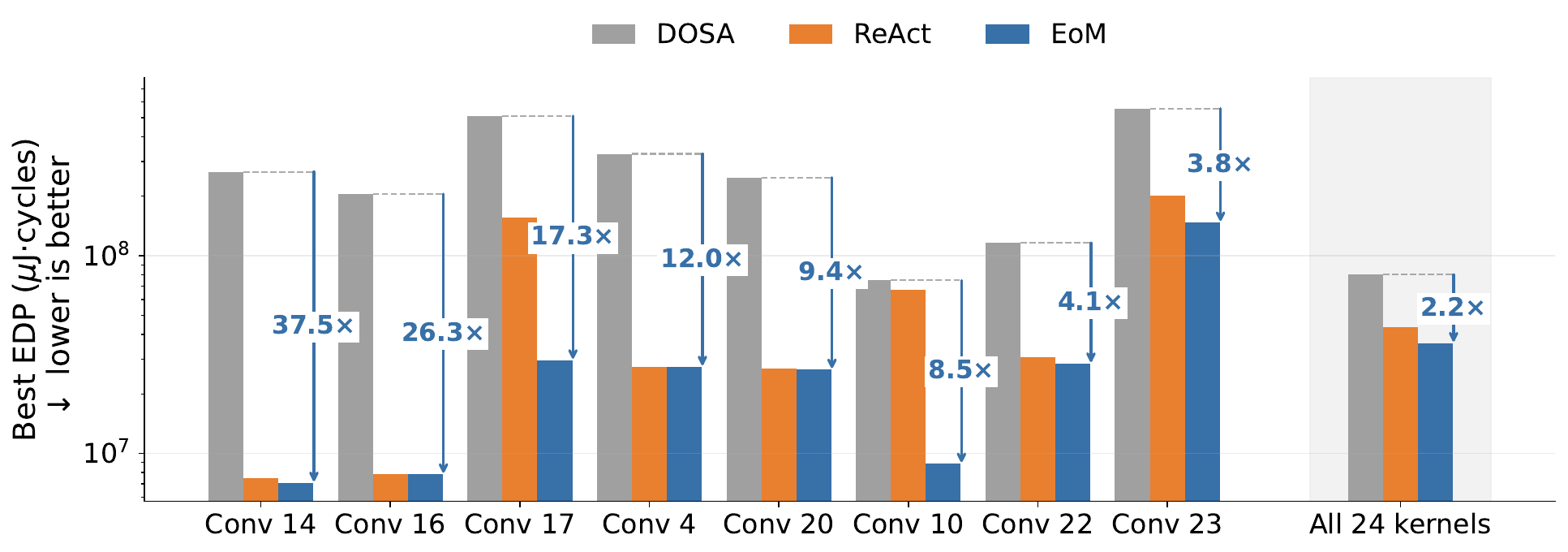}}
        {\fbox{\parbox[c][0.18\textheight][c]{0.9\linewidth}{\centering Placeholder for per-kernel accelerator comparison figure.}}}
        
        \footnotesize{(a) Per-kernel accelerator EDP}
    \end{minipage}
    \hfill
    \begin{minipage}{0.2\linewidth}
        \centering
        \includegraphics[width=0.9\linewidth]{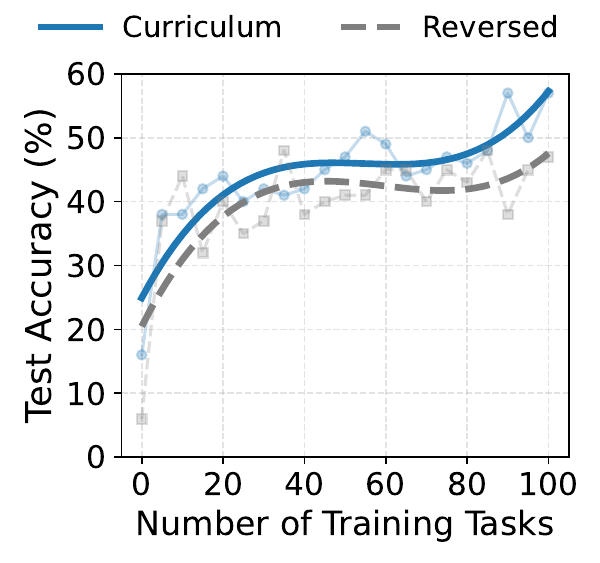}
        
        \footnotesize{(b) Curriculum learning on MATH}
    \end{minipage}
    \hfill
    \begin{minipage}{0.28\linewidth}
        \centering
        \includegraphics[width=0.9\linewidth]{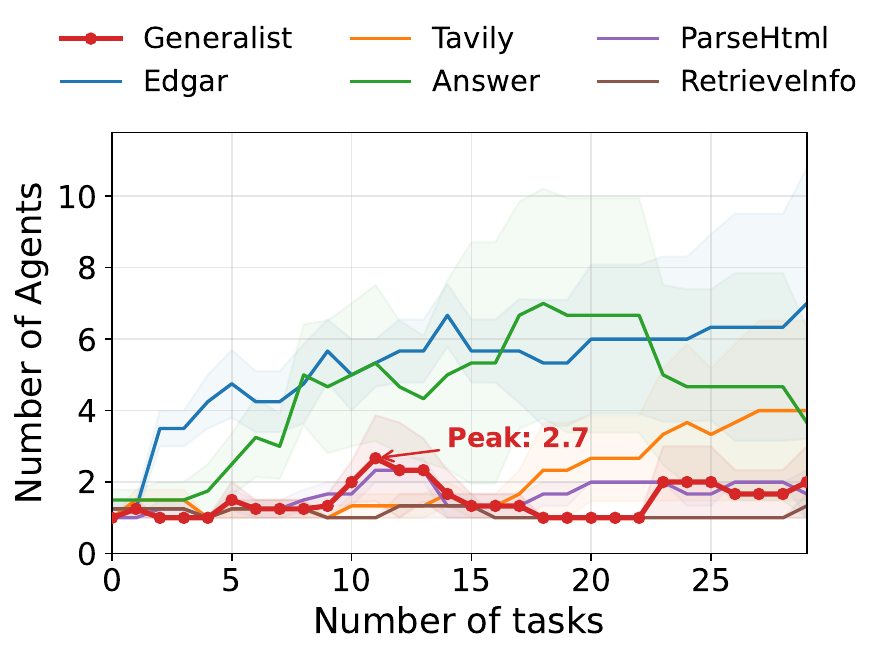}
        
        \footnotesize{(c) Finance research with a generalist}
    \end{minipage}

    \caption{\small \textbf{Mechanism, robustness, and generalization analyses.}
    \textbf{(a)} Per-kernel EDP on ResNet-50. Best EDP found by \textsc{DOSA}, \textsc{ReAct}, and \ours{} on log scale; lower is better.
    \textbf{(b)} Comparison between the default easy-to-hard curriculum and a reversed hard-to-easy curriculum.
    \textbf{(c)} Adding a strong generalist agent with access to all tools does not automatically dominate specialized agents.}
    \label{fig:mechanism-robustness}
\end{figure}

\textbf{Can a complete generalist monopolize the economy?}
A natural concern is that if a single agent is given complete access to the task interface, it may dominate the market and collapse the society into a monolithic system. We test this in Finance-Agent-Bench by adding a generalist agent with access to all tools alongside the partial specialist agents. As shown in \Cref{fig:mechanism-robustness} (c), the generalist briefly expands around tasks 11--12, but then contracts and returns to a single agent for the remainder of training. Meanwhile, specialized populations such as Edgar and Tavily continue to grow, reaching roughly 5--8 agents late in training.
This suggests that broader access does not automatically imply market dominance. The economy rewards local value: a specialist whose wake-up condition, tool use, and evidence standard are tuned to a narrow subproblem can outcompete a generalist whose prompt budget is spread across many heterogeneous responsibilities. We provide a detailed analysis in Appendix~\ref{sec:case-study-generalist-no-monopoly}, where we show that specialists evolve sharper local decision rules, while the generalist tends to accumulate broad but diluted procedural instructions. Thus, the society remains decentralized not because complete agents are forbidden, but because the market continues to favor locally more precise specialists.

\section{Related Work}

\textbf{Self-Evolving Agents.} Recent work studies how agents can improve through interaction, adaptation, and accumulated experience rather than relying on fixed centralized control. In the LLM setting, this includes co-evolution, iterative refinement, and experience sharing across agent populations \cite{chen2025multi,weng2026group,Dai2026GeoEvolver}. Related work on reward-driven self-organization shows that structured behaviors can emerge from local feedback alone \cite{zhou2025reso}. More broadly, rich ecological environments and evolutionary pressures have been shown to induce increasingly complex adaptive behaviors, suggesting that agent capabilities can emerge through continual interaction with other agents and the environment \cite{bejjani2025emergence}. For more related work, please refer to Appendix~\ref{sec.extended-rw}.

\section{Limitations, Conclusions, and Future Work}
\label{sec.conclusion}

This work shows that decentralized economic interactions can turn a population of agents into adaptive collective intelligence. Across diverse domains, auctions, transactions, and wealth-based selection induce specialization and coordination without centralized control, suggesting that designing incentives is a powerful alternative to designing individual agents.

A key limitation is that adaptation occurs only in prompt space with a frozen backbone, which may restrict capability growth in tasks requiring new skills or representations. Future work can extend this framework to parameter-space training, hybrid adaptation, and broader agent backbones, including multimodal or embodied systems.

\begin{ack}
    This work was supported in part by a gift from the Chan Zuckerberg Initiative Foundation to establish the Kempner Institute at Harvard University. 
\end{ack}

\bibliography{ref}

\begin{thebibliography}{56}
\providecommand{\natexlab}[1]{#1}
\providecommand{\url}[1]{\texttt{#1}}
\expandafter\ifx\csname urlstyle\endcsname\relax
  \providecommand{\doi}[1]{doi: #1}\else
  \providecommand{\doi}{doi: \begingroup \urlstyle{rm}\Url}\fi

\bibitem[Baum(1996)]{Baum1996Toward}
Eric~B. Baum.
\newblock Toward a model of mind as a laissez-faire economy of idiots.
\newblock In \emph{Proceedings of the 13th International Conference on Machine Learning}, pages 20--27, 1996.

\bibitem[Baum(1999)]{baum1999toward}
Eric~B Baum.
\newblock Toward a model of intelligence as an economy of agents.
\newblock \emph{Machine Learning}, 35\penalty0 (2):\penalty0 155--185, 1999.

\bibitem[Baum and Durdanovic(2000)]{baum2000evolution}
Eric~B Baum and Igor Durdanovic.
\newblock Evolution of cooperative problem solving in an artificial economy.
\newblock \emph{Neural Computation}, 12\penalty0 (12):\penalty0 2743--2775, 2000.

\bibitem[Bejjani et~al.(2025)Bejjani, Van~Amburg, Wang, Su, Pratt, Mazloumi, Khoshnevis, Kakade, Brantley, and Walsman]{bejjani2025emergence}
Joseph Bejjani, Chase Van~Amburg, Chengrui Wang, Chloe~Huangyuan Su, Sarah~M Pratt, Yasin Mazloumi, Naeem Khoshnevis, Sham~M Kakade, Kiant{\'e} Brantley, and Aaron Walsman.
\newblock The emergence of complex behavior in large-scale ecological environments.
\newblock \emph{arXiv preprint arXiv:2510.18221}, 2025.

\bibitem[Bigeard et~al.(2025)Bigeard, Nashold, Krishnan, and Wu]{bigeard2025finance}
Antoine Bigeard, Langston Nashold, Rayan Krishnan, and Shirley Wu.
\newblock Finance agent benchmark: Benchmarking llms on real-world financial research tasks.
\newblock \emph{arXiv preprint arXiv:2508.00828}, 2025.

\bibitem[Cemri et~al.(2025)Cemri, Pan, Yang, Agrawal, Chopra, Tiwari, Keutzer, Parameswaran, Klein, Ramchandran, et~al.]{cemri2025multi}
Mert Cemri, Melissa~Z Pan, Shuyi Yang, Lakshya~A Agrawal, Bhavya Chopra, Rishabh Tiwari, Kurt Keutzer, Aditya Parameswaran, Dan Klein, Kannan Ramchandran, et~al.
\newblock Why do multi-agent llm systems fail?
\newblock \emph{arXiv preprint arXiv:2503.13657}, 2025.

\bibitem[Chen et~al.(2025)Chen, Wang, Zhu, Yu, Feng, Zhang, Patwary, and You]{chen2025multi}
Yixing Chen, Yiding Wang, Siqi Zhu, Haofei Yu, Tao Feng, Muhan Zhang, Mostofa Patwary, and Jiaxuan You.
\newblock Multi-agent evolve: Llm self-improve through co-evolution.
\newblock \emph{arXiv preprint arXiv:2510.23595}, 2025.

\bibitem[Cheng et~al.(2025)Cheng, Liu, Pan, Li, Wang, Krentsel, Xia, Cemri, Park, Yang, et~al.]{cheng2025barbarians}
Audrey Cheng, Shu Liu, Melissa Pan, Zhifei Li, Bowen Wang, Alex Krentsel, Tian Xia, Mert Cemri, Jongseok Park, Shuo Yang, et~al.
\newblock Barbarians at the gate: How ai is upending systems research.
\newblock \emph{arXiv preprint arXiv:2510.06189}, 2025.

\bibitem[Dai et~al.(2026)Dai, Zhu, et~al.]{Dai2026GeoEvolver}
X.~Dai, Y.~Zhu, et~al.
\newblock Geoevolver: A self-evolving multi-agent system for earth observation.
\newblock \emph{arXiv preprint arXiv:2602.02559}, 2026.

\bibitem[Dias et~al.(2006)Dias, Zlot, Kalra, and Stentz]{dias2006market}
M~Bernardine Dias, Robert Zlot, Nidhi Kalra, and Anthony Stentz.
\newblock Market-based multirobot coordination: A survey and analysis.
\newblock \emph{Proceedings of the IEEE}, 94\penalty0 (7):\penalty0 1257--1270, 2006.

\bibitem[Du et~al.(2024)Du, Li, Torralba, Tenenbaum, and Mordatch]{du2024improving}
Yilun Du, Shuang Li, Antonio Torralba, Joshua~B Tenenbaum, and Igor Mordatch.
\newblock Improving factuality and reasoning in language models through multiagent debate.
\newblock In \emph{Forty-first international conference on machine learning}, 2024.

\bibitem[Duetting et~al.(2024)Duetting, Mirrokni, Paes~Leme, Xu, and Zuo]{duetting2024mechanism}
Paul Duetting, Vahab Mirrokni, Renato Paes~Leme, Haifeng Xu, and Song Zuo.
\newblock Mechanism design for large language models.
\newblock In \emph{Proceedings of the ACM Web Conference 2024}, pages 144--155, 2024.

\bibitem[Genc et~al.(2021)Genc, Kim, Amid, Haj-Ali, Iyer, Prakash, Zhao, Grubb, Liew, Mao, et~al.]{genc2021gemmini}
Hasan Genc, Seah Kim, Alon Amid, Ameer Haj-Ali, Vighnesh Iyer, Pranav Prakash, Jerry Zhao, Daniel Grubb, Harrison Liew, Howard Mao, et~al.
\newblock Gemmini: Enabling systematic deep-learning architecture evaluation via full-stack integration.
\newblock In \emph{2021 58th ACM/IEEE Design Automation Conference (DAC)}, pages 769--774. IEEE, 2021.

\bibitem[Hayek(1945)]{hayek1945knowledge}
Friedrich~A. Hayek.
\newblock The use of knowledge in society.
\newblock \emph{The American Economic Review}, 35\penalty0 (4):\penalty0 519--530, 1945.

\bibitem[Hayek(1973)]{Hayek1973Law}
Friedrich~A. Hayek.
\newblock \emph{Law, Legislation and Liberty}.
\newblock University of Chicago Press, 1973.

\bibitem[Hendrycks et~al.(2021)Hendrycks, Burns, Kadavath, Arora, Basart, Tang, Song, and Steinhardt]{hendrycks2021measuring}
Dan Hendrycks, Collin Burns, Saurav Kadavath, Akul Arora, Steven Basart, Eric Tang, Dawn Song, and Jacob Steinhardt.
\newblock Measuring mathematical problem solving with the math dataset.
\newblock \emph{arXiv preprint arXiv:2103.03874}, 2021.

\bibitem[Holland(1985)]{Holland1985Properties}
John~H. Holland.
\newblock Properties of the bucket brigade algorithm.
\newblock In \emph{Proceedings of the 1st International Conference on Genetic Algorithms}, 1985.

\bibitem[Hong et~al.(2023{\natexlab{a}})Hong, Huang, Dinh, Subedar, and Shao]{hong2023dosa}
Charles Hong, Qijing Huang, Grace Dinh, Mahesh Subedar, and Yakun~Sophia Shao.
\newblock Dosa: Differentiable model-based one-loop search for dnn accelerators.
\newblock In \emph{Proceedings of the 56th Annual IEEE/ACM International Symposium on Microarchitecture}, pages 209--224, 2023{\natexlab{a}}.

\bibitem[Hong et~al.(2023{\natexlab{b}})Hong, Zhuge, Chen, Zheng, Cheng, Wang, Zhang, Wang, Yau, Lin, et~al.]{hong2023metagpt}
Sirui Hong, Mingchen Zhuge, Jonathan Chen, Xiawu Zheng, Yuheng Cheng, Jinlin Wang, Ceyao Zhang, Zili Wang, Steven Ka~Shing Yau, Zijuan Lin, et~al.
\newblock Metagpt: Meta programming for a multi-agent collaborative framework.
\newblock In \emph{The twelfth international conference on learning representations}, 2023{\natexlab{b}}.

\bibitem[Iqbal et~al.(2022)Iqbal, Costales, and Sha]{iqbal2022alma}
Shariq Iqbal, Robby Costales, and Fei Sha.
\newblock Alma: Hierarchical learning for composite multi-agent tasks.
\newblock \emph{Advances in neural information processing systems}, 35:\penalty0 7155--7166, 2022.

\bibitem[Kim et~al.(2024)Kim, Pertsch, Karamcheti, Xiao, Balakrishna, Nair, Rafailov, Foster, Lam, Sanketi, et~al.]{kim2024openvla}
Moo~Jin Kim, Karl Pertsch, Siddharth Karamcheti, Ted Xiao, Ashwin Balakrishna, Suraj Nair, Rafael Rafailov, Ethan Foster, Grace Lam, Pannag Sanketi, et~al.
\newblock Openvla: An open-source vision-language-action model.
\newblock \emph{arXiv preprint arXiv:2406.09246}, 2024.

\bibitem[Kim et~al.(2025)Kim, Gu, Park, Park, Schmidgall, Heydari, Yan, Zhang, Zhuang, Liu, et~al.]{kim2025towards}
Yubin Kim, Ken Gu, Chanwoo Park, Chunjong Park, Samuel Schmidgall, A~Ali Heydari, Yao Yan, Zhihan Zhang, Yuchen Zhuang, Yun Liu, et~al.
\newblock Towards a science of scaling agent systems.
\newblock \emph{arXiv preprint arXiv:2512.08296}, 2025.

\bibitem[Lifshitz et~al.(2025)Lifshitz, McIlraith, and Du]{lifshitz2025multi}
Shalev Lifshitz, Sheila~A McIlraith, and Yilun Du.
\newblock Multi-agent verification: Scaling test-time compute with multiple verifiers.
\newblock \emph{arXiv preprint arXiv:2502.20379}, 2025.

\bibitem[Liu et~al.()Liu, Chen, and Amato]{liuimproved}
Shuo Liu, Tianle Chen, and Christopher Amato.
\newblock Improved multi-agent collaboration with multi-turn reinforcement learning.
\newblock In \emph{First Workshop on Multi-Turn Interactions in Large Language Models}.

\bibitem[Long et~al.(2026)Long, Simchi-Levi, Zhu, Su, Calmon, and Calmon]{long2026reliability}
Carol~Xuan Long, David Simchi-Levi, Feng Zhu, Huangyuan Su, Andre~P. Calmon, and Flavio~P. Calmon.
\newblock Reliability and effectiveness of autonomous ai agents in supply chain management.
\newblock \emph{arXiv preprint arXiv:2605.17036}, 2026.

\bibitem[Minsky(1988)]{Minsky1988Society}
Marvin Minsky.
\newblock \emph{The Society of Mind}.
\newblock Simon and Schuster, 1988.

\bibitem[Novikov et~al.(2025)Novikov, V{\~u}, Eisenberger, Dupont, Huang, Wagner, Shirobokov, Kozlovskii, Ruiz, Mehrabian, et~al.]{novikov2025alphaevolve}
Alexander Novikov, Ng{\^a}n V{\~u}, Marvin Eisenberger, Emilien Dupont, Po-Sen Huang, Adam~Zsolt Wagner, Sergey Shirobokov, Borislav Kozlovskii, Francisco~JR Ruiz, Abbas Mehrabian, et~al.
\newblock Alphaevolve: A coding agent for scientific and algorithmic discovery.
\newblock \emph{arXiv preprint arXiv:2506.13131}, 2025.

\bibitem[Omidshafiei et~al.(2017)Omidshafiei, Pazis, Amato, How, and Vian]{omidshafiei2017deep}
Shayegan Omidshafiei, Jason Pazis, Christopher Amato, Jonathan~P How, and John Vian.
\newblock Deep decentralized multi-task multi-agent reinforcement learning under partial observability.
\newblock In \emph{International conference on machine learning}, pages 2681--2690. PMLR, 2017.

\bibitem[Park et~al.(2025)Park, Han, Guo, Ozdaglar, Zhang, and Kim]{park2025maporl}
Chanwoo Park, Seungju Han, Xingzhi Guo, Asuman~E Ozdaglar, Kaiqing Zhang, and Joo-Kyung Kim.
\newblock Maporl: Multi-agent post-co-training for collaborative large language models with reinforcement learning.
\newblock In \emph{Proceedings of the 63rd Annual Meeting of the Association for Computational Linguistics (Volume 1: Long Papers)}, pages 30215--30248, 2025.

\bibitem[Patil et~al.(2024)Patil, Zhang, Wang, and Gonzalez]{patil2024gorilla}
Shishir~G Patil, Tianjun Zhang, Xin Wang, and Joseph~E Gonzalez.
\newblock Gorilla: Large language model connected with massive apis.
\newblock \emph{Advances in Neural Information Processing Systems}, 37:\penalty0 126544--126565, 2024.

\bibitem[Prakash et~al.(2025)Prakash, Cheng, Tschand, Mazumder, Gohil, Ma, Yik, Wan, Quaye, Alvanaki, et~al.]{prakash2025quarch}
Shvetank Prakash, Andrew Cheng, Arya Tschand, Mark Mazumder, Varun Gohil, Jeffrey Ma, Jason Yik, Zishen Wan, Jessica Quaye, Elisavet~Lydia Alvanaki, et~al.
\newblock Quarch: A benchmark for evaluating llm reasoning in computer architecture.
\newblock \emph{arXiv preprint arXiv:2510.22087}, 2025.

\bibitem[Qian et~al.(2024)Qian, Liu, Liu, Chen, Dang, Li, Yang, Chen, Su, Cong, et~al.]{qian2024chatdev}
Chen Qian, Wei Liu, Hongzhang Liu, Nuo Chen, Yufan Dang, Jiahao Li, Cheng Yang, Weize Chen, Yusheng Su, Xin Cong, et~al.
\newblock Chatdev: Communicative agents for software development.
\newblock In \emph{Proceedings of the 62nd annual meeting of the association for computational linguistics (volume 1: Long papers)}, pages 15174--15186, 2024.

\bibitem[Qu et~al.(2026)Qu, Zheng, Zhou, Yan, Tang, Ong, Hong, Zhou, Jiang, Kong, et~al.]{qu2026coral}
Ao~Qu, Han Zheng, Zijian Zhou, Yihao Yan, Yihong Tang, Shao~Yong Ong, Fenglu Hong, Kaichen Zhou, Chonghe Jiang, Minwei Kong, et~al.
\newblock Coral: Towards autonomous multi-agent evolution for open-ended discovery.
\newblock \emph{arXiv preprint arXiv:2604.01658}, 2026.

\bibitem[Schmidhuber(1989)]{Schmidhuber1989NeuralBucket}
J{\"u}rgen Schmidhuber.
\newblock The neural bucket brigade: A local learning algorithm for dynamic feedforward and recurrent networks.
\newblock \emph{Connection Science}, 1\penalty0 (4):\penalty0 403--412, 1989.

\bibitem[Sharma(2025)]{openevolve}
Asankhaya Sharma.
\newblock Openevolve: an open-source evolutionary coding agent, 2025.
\newblock URL \url{https://github.com/algorithmicsuperintelligence/openevolve}.

\bibitem[Subramaniam et~al.(2025)Subramaniam, Du, Tenenbaum, Torralba, Li, and Mordatch]{subramaniam2025multiagent}
Vighnesh Subramaniam, Yilun Du, Joshua~B Tenenbaum, Antonio Torralba, Shuang Li, and Igor Mordatch.
\newblock Multiagent finetuning: Self improvement with diverse reasoning chains.
\newblock \emph{arXiv preprint arXiv:2501.05707}, 2025.

\bibitem[Sudhir and Tran-Thanh(2025)]{Sudhir2025MarketBased}
Abhimanyu~Pallavi Sudhir and Long Tran-Thanh.
\newblock Market-based architectures in rl and beyond.
\newblock \emph{Accepted to AAMAS 2025 Blue Sky Track}, abs/2503.05828, 2025.

\bibitem[Sun et~al.(2022)Sun, Wang, Zhu, Ning, Dai, Yang, and Wang]{sun2022gibbon}
Hanbo Sun, Chenyu Wang, Zhenhua Zhu, Xuefei Ning, Guohao Dai, Huazhong Yang, and Yu~Wang.
\newblock Gibbon: Efficient co-exploration of nn model and processing-in-memory architecture.
\newblock In \emph{2022 Design, Automation \& Test in Europe Conference \& Exhibition (DATE)}, pages 867--872. IEEE, 2022.

\bibitem[Sun et~al.(2025)Sun, Lu, Ling, Liu, Yao, Yang, and Chen]{sun2025scaling}
Weiwei Sun, Miao Lu, Zhan Ling, Kang Liu, Xuesong Yao, Yiming Yang, and Jiecao Chen.
\newblock Scaling long-horizon llm agent via context-folding.
\newblock \emph{arXiv preprint arXiv:2510.11967}, 2025.

\bibitem[Tschand et~al.(2026)Tschand, Wang, Wan, Cheng, Cristescu, He, Huang, Ingare, Kangaslahti, Kangaslahti, Lebryk, Lin, Ma, Meterez, Mohri, Morwani, Qin, Rinberg, Rodriguez-Diaz, Taliotis, Fathi, Zhao, Zhou, and Reddi]{tschand2026genaisystemsrecurringchallenges}
Arya Tschand, Chenyu Wang, Zishen Wan, Andrew Cheng, Ioana Cristescu, Kevin He, Howard Huang, Alexander Ingare, Akseli Kangaslahti, Sara Kangaslahti, Theo Lebryk, Hongjin Lin, Jeffrey~Jian Ma, Alexandru Meterez, Clara Mohri, Depen Morwani, Sunny Qin, Roy Rinberg, Paula Rodriguez-Diaz, Alyssa~Mia Taliotis, Pernille~Undrum Fathi, Rosie Zhao, Todd Zhou, and Vijay~Janapa Reddi.
\newblock Genai for systems: Recurring challenges and design principles from software to silicon, 2026.
\newblock URL \url{https://arxiv.org/abs/2602.15241}.

\bibitem[Wang et~al.(2025{\natexlab{a}})Wang, Wan, Kang, Chen, Xie, Krishna, Reddi, and Du]{wang2025slm}
Chenyu Wang, Zishen Wan, Hao Kang, Emma Chen, Zhiqiang Xie, Tushar Krishna, Vijay~Janapa Reddi, and Yilun Du.
\newblock Slm-mux: Orchestrating small language models for reasoning.
\newblock \emph{arXiv preprint arXiv:2510.05077}, 2025{\natexlab{a}}.

\bibitem[Wang(2026)]{AESP2026}
Jian~Sheng Wang.
\newblock Aesp: A human-sovereign economic protocol for ai agents with privacy-preserving settlement.
\newblock \emph{arXiv preprint arXiv:2603.00318}, 2026.

\bibitem[Wang et~al.(2026)Wang, Lin, Hu, Jiao, Chowdhury, Chang, and Patwardhan]{wang2026frontierscience}
Miles Wang, Robi Lin, Kat Hu, Joy Jiao, Neil Chowdhury, Ethan Chang, and Tejal Patwardhan.
\newblock Frontierscience: Evaluating ai's ability to perform expert-level scientific tasks.
\newblock \emph{arXiv preprint arXiv:2601.21165}, 2026.

\bibitem[Wang et~al.(2025{\natexlab{b}})Wang, Su, Zeng, Xu, Ren, Yang, Huang, He, Ma, Peng, et~al.]{wang2025thetaevolve}
Yiping Wang, Shao-Rong Su, Zhiyuan Zeng, Eva Xu, Liliang Ren, Xinyu Yang, Zeyi Huang, Xuehai He, Luyao Ma, Baolin Peng, et~al.
\newblock Thetaevolve: Test-time learning on open problems.
\newblock \emph{arXiv preprint arXiv:2511.23473}, 2025{\natexlab{b}}.

\bibitem[Weng et~al.(2026)Weng, Antoniades, Nathani, Zhang, Pu, and Wang]{weng2026group}
Zhaotian Weng, Antonis Antoniades, Deepak Nathani, Zhen Zhang, Xiao Pu, and Xin~Eric Wang.
\newblock Group-evolving agents: Open-ended self-improvement via experience sharing.
\newblock \emph{arXiv preprint arXiv:2602.04837}, 2026.

\bibitem[Wu et~al.(2024)Wu, Bansal, Zhang, Wu, Li, Zhu, Jiang, Zhang, Zhang, Liu, et~al.]{wu2024autogen}
Qingyun Wu, Gagan Bansal, Jieyu Zhang, Yiran Wu, Beibin Li, Erkang Zhu, Li~Jiang, Xiaoyun Zhang, Shaokun Zhang, Jiale Liu, et~al.
\newblock Autogen: Enabling next-gen llm applications via multi-agent conversations.
\newblock In \emph{First conference on language modeling}, 2024.

\bibitem[Xu(2026)]{AgentEconomy2026}
Minghui Xu.
\newblock The agent economy: A blockchain-based foundation for autonomous ai agents.
\newblock 2026.
\newblock NeurIPS 2026.

\bibitem[Xue et~al.(2025)Xue, Zhou, Zhang, Zhang, Li, Zhang, Yin, Torr, Ouyang, and Bai]{xue2025comas}
Xiangyuan Xue, Yifan Zhou, Guibin Zhang, Zaibin Zhang, Yijiang Li, Chen Zhang, Zhenfei Yin, Philip Torr, Wanli Ouyang, and Lei Bai.
\newblock Comas: Co-evolving multi-agent systems via interaction rewards.
\newblock \emph{arXiv preprint arXiv:2510.08529}, 2025.

\bibitem[Yang et~al.(2025)Yang, Chai, Shao, Song, Qi, Rui, and Zhang]{agentnet}
Yingxuan Yang, Huacan Chai, Shuai Shao, Yuanyi Song, Siyuan Qi, Renting Rui, and Weinan Zhang.
\newblock Agentnet: Decentralized evolutionary coordination for llm-based multi-agent systems, 2025.
\newblock URL \url{https://arxiv.org/abs/2504.00587}.

\bibitem[Yao et~al.(2022)Yao, Zhao, Yu, Du, Shafran, Narasimhan, and Cao]{yao2022react}
Shunyu Yao, Jeffrey Zhao, Dian Yu, Nan Du, Izhak Shafran, Karthik~R Narasimhan, and Yuan Cao.
\newblock React: Synergizing reasoning and acting in language models.
\newblock In \emph{The eleventh international conference on learning representations}, 2022.

\bibitem[Ye et~al.(2026)Ye, Ge, Zheng, Gao, Yu, Kurian, Indupuru, Tan, Zhu, Xiang, et~al.]{ye2026world}
Seonghyeon Ye, Yunhao Ge, Kaiyuan Zheng, Shenyuan Gao, Sihyun Yu, George Kurian, Suneel Indupuru, You~Liang Tan, Chuning Zhu, Jiannan Xiang, et~al.
\newblock World action models are zero-shot policies.
\newblock \emph{arXiv preprint arXiv:2602.15922}, 2026.

\bibitem[Yuksekgonul et~al.(2026)Yuksekgonul, Koceja, Li, Bianchi, McCaleb, Wang, Kautz, Choi, Zou, Guestrin, and Sun]{ttt-discover2026}
Mert Yuksekgonul, Daniel Koceja, Xinhao Li, Federico Bianchi, Jed McCaleb, Xiaolong Wang, Jan Kautz, Yejin Choi, James Zou, Carlos Guestrin, and Yu~Sun.
\newblock Learning to discover at test time.
\newblock \emph{arXiv preprint}, 2026.

\bibitem[Zhang et~al.(2026)Zhang, Kim, Xiang, Gao, and Cao]{zhang2026dynamic}
Miao Zhang, Junsik Kim, Siyuan Xiang, Jian Gao, and Cheng Cao.
\newblock Dynamic role assignment for multi-agent debate.
\newblock \emph{arXiv preprint arXiv:2601.17152}, 2026.

\bibitem[Zhou et~al.(2025)Zhou, Geng, Xue, Kang, Qin, Wang, Yin, and Bai]{zhou2025reso}
Heng Zhou, Hejia Geng, Xiangyuan Xue, Li~Kang, Yiran Qin, Zhiyong Wang, Zhenfei Yin, and Lei Bai.
\newblock Reso: A reward-driven self-organizing llm-based multi-agent system for reasoning tasks.
\newblock In \emph{Proceedings of the 2025 Conference on Empirical Methods in Natural Language Processing}, pages 15990--16009, 2025.

\bibitem[Zhu and Du(2025)]{zhu2025role}
Andy Zhu and Yingjun Du.
\newblock A role-aware multi-agent framework for financial education question answering with llms.
\newblock \emph{arXiv preprint arXiv:2509.09727}, 2025.

\bibitem[Zhuge et~al.(2023)Zhuge, Liu, Faccio, Ashley, Csord{\'a}s, Gopalakrishnan, Hamdi, Hammoud, Herrmann, Irie, et~al.]{zhuge2023mindstorms}
Mingchen Zhuge, Haozhe Liu, Francesco Faccio, Dylan~R Ashley, R{\'o}bert Csord{\'a}s, Anand Gopalakrishnan, Abdullah Hamdi, Hasan Abed Al~Kader Hammoud, Vincent Herrmann, Kazuki Irie, et~al.
\newblock Mindstorms in natural language-based societies of mind.
\newblock \emph{arXiv preprint arXiv:2305.17066}, 2023.

\end{thebibliography}
\bibliographystyle{plainnat}


\newpage

\appendix

\etocdepthtag.toc{appendix}
\etocsettocstyle{\section*{Appendix - Table of Contents}}{}
\tableofcontents
\vspace{1em}

\clearpage
\section{Extended Implementation Details}
\label{sec:pseudo_code}
Training and evaluation pseudo code for \ours{} are shown below: 
\begin{algorithm}[ht!]
\caption{The Training Loop}
\label{alg:training}
\begin{algorithmic}[1]
\Require task stream $\mathcal{D}$, initial population $\mathcal{P}_0$, wake-up oracle $\omega$
\Require base bid $b_0$, novice premium $\epsilon$, step cap $S$, trial cap $T$
\Require bounds $[N_{\min}, N_{\max}]$; rent $\rho$ every $K_r$ tasks
\Require birth probs $p_a, p_b$ (bankruptcy), $p_g$ (periodic); period $K_b$, batch $B$
\State $\mathcal{P} \gets \mathcal{P}_0$
\For{episode $e = 1, 2, \dots$ over $\tau \sim \mathcal{D}$}
    \State $\sigma \gets \{a \mapsto \textsc{Snapshot}(a) : a \in \mathcal{P}\}$;\quad $\mathcal{B}_e \gets \emptyset$ \Comment{e.g.\ wealth, capability, \dots}
    \For{trial $t = 1, \dots, T$} \Comment{replay until no bankruptcy}
        \State $a^{\text{last}}, H \gets$ \Call{Planning}{$\tau, \mathcal{P}$};\; $a^{\text{last}}.\text{wealth} \mathrel{+}= \text{env.reward}()$ \Comment{outcome reward to final actor}
        \State $\mathcal{B} \gets \{a : a.\text{wealth} \le 0\}$
        \If{$\mathcal{B} = \emptyset$} \textbf{break} \EndIf
        \State $\mathcal{P} \gets \mathcal{P} \setminus \mathcal{B}$;\; $\mathcal{B}_e \mathrel{\cup}= \mathcal{B}$
        \ForAll{$a \in \mathcal{P}$} \Call{Restore}{$a, \sigma(a)$} \EndFor \Comment{rollback survivors}
    \EndFor
    \If{$e \bmod K_r = 0$} \Comment{periodic rent}
        \State $a.\text{wealth} \mathrel{-}= r$ for all $a \in \mathcal{P}$
        \State $\mathcal{B}' \gets \{a \in \mathcal{P} : a.\text{wealth} \le 0\}$;\; $\mathcal{P} \gets \mathcal{P} \setminus \mathcal{B}'$;\; $\mathcal{B}_e \mathrel{\cup}= \mathcal{B}'$
    \EndIf
    \State \Call{Adaptation}{$\mathcal{P}, \mathcal{B}_e, e$} \Comment{bankruptcy + periodic births, replenish to $N_{\min}$}
\EndFor
\Statex
\Procedure{Planning}{$\tau, \mathcal{P}$}
    \State init env on $\tau$;\; $H \gets \emptyset$;\; $a^{-} \gets \bot$
    \For{$s = 1, \dots, S$}
        \State $\mathcal{E} \gets \{a \in \mathcal{P} : \omega(a, \text{env})\}$;\; \textbf{if} $\mathcal{E} = \emptyset$ \textbf{break}
        \State \Call{SetBids}{$\mathcal{E}$} \Comment{pluggable: e.g.\ novices enter at $\max_{\text{vet}} b + \epsilon$ and lock}
        \State $a^{*} \gets \arg\max_{a \in \mathcal{E}} b_a$ (random tie-break)
        \State $a^{*}.\text{wealth} \mathrel{-}= b_{a^{*}}$;\; \textbf{if} $a^{-} \neq \bot$: $a^{-}.\text{wealth} \mathrel{+}= b_{a^{*}}$
        \State env.apply($a^{*}.\text{act}(\text{env})$);\; $H \gets H \cup \{a^{*}\}$;\; $a^{-} \gets a^{*}$
        \State \textbf{if} env terminated \textbf{break}
    \EndFor
    \State \Return $a^{-}, H$
\EndProcedure
\Statex
\Procedure{Adaptation}{$\mathcal{P}, \mathcal{B}_e, e$}
    \ForAll{$a \in \mathcal{B}_e$ with $|\mathcal{P}| < N_{\max}$}
        \State $u \sim \mathcal{U}(0,1)$
        \State \textbf{if} $u < p_a$: $\mathcal{P} \gets \mathcal{P} \cup \{\textsc{Mutate}(\text{richest}(\mathcal{P}))\}$
        \State \textbf{elif} $u < p_a + p_b$: $\mathcal{P} \gets \mathcal{P} \cup \{\textsc{Amend}(a)\}$
    \EndFor
    \If{$e \bmod K_b = 0$}
        \For{$i = 1, \dots, B$ with $|\mathcal{P}| < N_{\max}$}
            \State $a' \gets \textsc{Mutate}(\text{richest}(\mathcal{P}))$ with prob.\ $p_g$, else $\textsc{Amend}(\text{poorest}(\mathcal{P}))$
            \State $\mathcal{P} \gets \mathcal{P} \cup \{a'\}$
        \EndFor
    \EndIf
    \While{$|\mathcal{P}| < N_{\min}$}
        \State $\mathcal{P} \gets \mathcal{P} \cup \{\textsc{Mutate}(a_0)\}$ for some template $a_0 \in \mathcal{P}_0$
    \EndWhile
\EndProcedure
\end{algorithmic}
\end{algorithm}
\begin{algorithm}[ht!]
    \caption{The Evaluation Loop}
    \label{alg:evaluation}
    \begin{algorithmic}[1]
    \Require trained population $\mathcal{P}^{*}$, test stream $\mathcal{D}_{\text{test}}$, wake-up oracle $\omega$
    \Require step cap $S$; thread-pool size $W$
    \State $\mathcal{P} \gets \mathcal{P}^{*}$;\; freeze $\{w_a, b_a, \text{status}_a\}_{a \in \mathcal{P}}$ \Comment{no payments, no rewards, no rent, no births}
    \For{task $\tau \in \mathcal{D}_{\text{test}}$ \textbf{in parallel} (workers $W$)}
        \State $\mathcal{P}_{\tau} \gets$ \Call{Snapshot}{$\mathcal{P}$} \Comment{thread-local deep copy; no shared mutation}
        \State $a^{\text{last}}, H \gets$ \Call{Planning}{$\tau, \mathcal{P}_{\tau}$} \Comment{no $a^{\text{last}}.\text{wealth}$ update}
        \State record $\{\tau, H, \text{env.terminated}, \text{env.score}()\}$
    \EndFor
    \Statex
    \Procedure{Planning}{$\tau, \mathcal{P}$}
        \State init env on $\tau$;\; $H \gets \emptyset$;\; $a^{-} \gets \bot$
        \For{$s = 1, \dots, S$}
            \State $\mathcal{E} \gets \{a \in \mathcal{P} : \omega(a, \text{env})\}$;\quad \textbf{if} $\mathcal{E} = \emptyset$ \textbf{break}
            \State \textit{(skip \Call{SetBids}{$\mathcal{E}$}: bids $b_a$ are frozen from training; no novice $\to$ veteran promotion)}
            \State $a^{*} \gets \arg\max_{a \in \mathcal{E}} b_a$ (random tie-break)
            \State \textit{(skip wealth transfer: $a^{*}.\text{wealth} \mathrel{-}= b_{a^{*}}$ and $a^{-}.\text{wealth} \mathrel{+}= b_{a^{*}}$ are no-ops)}
            \State env.apply($a^{*}.\text{act}(\text{env})$);\; $H \gets H \cup \{a^{*}\}$;\; $a^{-} \gets a^{*}$ \Comment{no reward credited to $a^{*}$}
            \State \textbf{if} env terminated \textbf{break}
        \EndFor
        \State \Return $a^{-}, H$
    \EndProcedure
    \Statex
    \end{algorithmic}
\end{algorithm}

\paragraph{Setting the bids.}
In the released \ours{}, bids are never updated online.
A new agent $a'$ receives its bid once, at the first step at which it is eligible, via the novice rule \eqref{eq:novice_rule}: $b_{a'}=\max_{a\in C_t} b_a+\varepsilon_{a'}$.
The bid is then frozen for the rest of training and at evaluation (\Cref{alg:training,alg:evaluation}).
The premium $\varepsilon_{a'}$ only ensures that the newcomer wins its first auction and is thus tested at least once. We also experimented with other bid-setting schemes, including LLM-generated bids, a single constant bid shared by all agents, and Holland-style wealth-proportional bids $b_a=\alpha W_a$ \citep{Holland1985Properties}.
The fixed novice rule above worked best in practice, which is consistent with Baum's finding that adaptive bid rules can produce inflationary speculative bubbles \citep{baum1999toward}.

\paragraph{Bids approximate values.}
Although no individual bid ever changes, the bids of the surviving population approximate the value of acting.
By transfer rule \eqref{eq:payments}, an agent $a$ that wins at context $x$ pays $b_a$ and receives the immediate reward plus the next winner's bid, so its expected profit per win is $V(a,x)-b_a$, where $V(a,x)$ is the market value of the state its action reaches (Appendix~\ref{sec:theory-selection}).
Agents that overbid ($b_a>V(a,x)$) go bankrupt.
Agents that underbid are profitable and tend to be mutated; their offspring re-enter at the incumbent's bid plus $\varepsilon$ and take over its wins.
Bids therefore increase until further entry is unprofitable, so the surviving bids approach the value from below, $V^\star(x)-\varepsilon_{\max}\le\beta_\infty(x)\le V^\star(x)$ (Appendix~\ref{thm:selection}).
The approximation is coarse: its resolution is $\varepsilon_{\max}$, rent keeps bids slightly below value, and each bid averages over instances on which the agent wins.

\section{Extended Related Works}
\label{sec.extended-rw}
\textbf{Multi-Agent Systems.} A parallel line of work examines how multiple agents coordinate, specialize, and solve tasks collectively. Recent work in the LLM era frames multi-agent orchestration as an ``society of minds,'' where specialization and coordination emerge through local interactions rather than centralized planning \cite{dias2006market,Sudhir2025MarketBased,AgentEconomy2026}. 
Existing approaches include reinforcement-learning-based collaboration \cite{liuimproved,park2025maporl}, dynamic role assignment and role-aware specialization \cite{zhang2026dynamic,zhu2025role,long2026reliability}, and decentralized evolutionary coordination, where agents dynamically adapt their graph connectivity, specialize through retrieval-augmented memory, and route tasks without a central orchestrator \cite{agentnet}. Other directions include multi-agent debate and verification for improving reasoning and test-time scaling \cite{du2024improving,lifshitz2025multi,wang2025slm}, and multi-agent finetuning with diverse reasoning chains for self-improvement \cite{subramaniam2025multiagent}, alongside broader studies of collaborative and self-organizing multi-agent systems. 
This perspective also connects to mechanism design for LLMs, which studies how incentives and allocation rules shape agent behavior in strategic environments \cite{duetting2024mechanism}.

\textbf{Economics and Applications to Agents.} The motivation for decentralized agent systems draws on a long tradition in economics and distributed intelligence. Smith’s ``invisible hand'' and Hayek’s ``spontaneous order'' emphasize how coordinated outcomes can emerge from decentralized actors with only local information \cite{Hayek1973Law}. Related ideas appear in Minsky’s \emph{Society of Mind}, which views intelligence as arising from the interaction of many simple agents rather than a single centralized reasoner \cite{Minsky1988Society, zhuge2023mindstorms}. These ideas influenced computational models of coordination and learning, including Holland’s bucket brigade for distributed credit assignment \cite{Holland1985Properties}, Schmidhuber’s neural bucket brigade \cite{Schmidhuber1989NeuralBucket}, and Baum’s Hayek Machine, which introduced property rights, wealth accumulation, and competitive selection among simple agents \cite{Baum1996Toward,baum1999toward,baum2000evolution}. Recent proposals for economic protocols and settlement layers for autonomous AI agents continue this line by formalizing exchange, sovereignty, and privacy-preserving settlement in modern agent ecosystems \cite{AESP2026}.
\section{Theoretical Motivations}
\label{sec:theory}

We use the notation and auction dynamics from \Cref{sec:method}. Agents
are prompted LLM policies with fixed bids and wealths, and wealth evolves
through the bucket-brigade transfer rule
\eqref{eq:payments}. This section summarizes
why the resulting market dynamics select useful specialists, why
outcome-only reward can be sufficient once the population is strong, and
how bucket-brigade payments provide a structured credit-assignment
signal. 

\subsection{Market selection drives bids toward value}
\label{sec:theory-selection}

The first result formalizes the basic market-selection mechanism. If an
agent repeatedly wins at a recurrent context and its bid is below its
expected resale-adjusted payoff, then its wealth has positive drift and it
can survive. Conversely, agents whose bids exceed their expected payoff
eventually become insolvent. Thus bankruptcy and replacement push the
surviving bid frontier toward the value of the best available specialist.

For a recurrent context \(x\), let \(V(a,x)\) denote the expected
resale-adjusted payoff of agent \(a\), including the immediate reward and
the downstream bucket-brigade payment. Let
\[
  V^\star(x):=\sup_{a:\phi_a(x)=1} V(a,x)
\]
be the value of the best eligible specialist, and let
\(\beta_\infty(x)\) denote the largest bid among agents that win at
\(x\) infinitely often and are never removed.

\begin{theorem}[Selection of bids toward expected payoff]
\label{thm:selection}
Suppose context \(x\) is recurrent, payoffs are stationary and bounded,
and newly injected agents use the novice bidding rule with resolution
\(\varepsilon_{\max}\). Suppose also that whenever the current solvent
frontier at \(x\) is more than \(\varepsilon_{\max}\) below \(V^\star(x)\),
there is a positive probability of injecting an eligible near-optimal
specialist. Then, almost surely,
\[
  V^\star(x)-\varepsilon_{\max}
  \le
  \beta_\infty(x)
  \le
  V^\star(x).
\]
Thus the long-run solvent bid frontier tracks the value of the best
prompted specialist, up to the novice-resolution error.
\end{theorem}

This result explains the role of wealth and bankruptcy: the mechanism
does not require explicitly fitting value functions. Instead, agents with
overpriced bids lose wealth and are removed, while underpriced useful
agents have a chance to accumulate wealth and persist.

\subsection{Outcome-only reward is sufficient for a strong population}
\label{sec:outcome-suffices}

A natural question is whether the system requires dense process rewards,
or whether final outcome reward alone is enough. The next result shows
that outcome-only reward is sufficient whenever the auction population is
already organized so that the winning agent is approximately the best
eligible specialist at each reachable history.

Let \(J^{\mathrm{out}}(\pi)\) denote the expected discounted outcome
return of a policy \(\pi\), using only the task-aligned reward
\(r_t^{\mathrm{out}}\). Let \(Q_t^\star(h_t,u)\) be the optimal
outcome-value obtained by selecting agent \(u\) at history \(h_t\) and
acting optimally thereafter.

\begin{theorem}[Outcome reward suffices for strong agents]
\label{thm:outcome-suffices}
Suppose there exists \(\varepsilon\ge 0\) such that for every reachable
history \(h_t\) with nonempty eligible set \(E_t\), the auction winner
\(a_t^\star\) satisfies
\[
  Q_t^\star(h_t,a_t^\star)
  \ge
  \max_{u\in E_t} Q_t^\star(h_t,u) - \varepsilon.
\]
Then with the usual interpretation that the sum equals \(H\) when
\(\gamma=1\), the induced auction policy \(\pi^{\mathrm{auc}}\) satisfies
\[
  J^{\mathrm{out}}(\pi^{\mathrm{auc}})
  \ge
  \sup_{\pi} J^{\mathrm{out}}(\pi)
  -
  \varepsilon\,\frac{1-\gamma^H}{1-\gamma},
\]

\end{theorem}

The theorem separates two questions. Process rewards may help train or
shape the population, but they are not required for correctness once the
market reliably selects near-best specialists under the outcome
objective.

\subsection{Regret to an oracle coordinator}
\label{sec:oracle-regret}

We next compare the decentralized auction to a centralized oracle that
directly selects the best eligible specialist at every history. The
result shows that if market bids approximate the oracle action-values,
then the auction policy has vanishing regret relative to this coordinator.

\begin{theorem}[Vanishing regret of the market auction policy]
\label{thm:oracle-regret-decay}
Suppose there exists a nonincreasing sequence \(\{\beta_e\}_{e\ge 1}\)
such that for every episode \(e\), reachable history \(h_t\), and
eligible agent \(u\in E_{e,t}(h_t)\),
\[
  \left| b_{e,t}(h_t,u) - Q^\star_{e,t}(h_t,u) \right|
  \le
  \beta_e.
\]
Let \(r_e\) be the episode regret of the market policy relative to the
oracle coordinator. Then
\(
  r_e
  \le
  2\beta_e \sum_{t=0}^{H-1}\gamma^t.
\)
Consequently,
\(
  \mathrm{Reg}(E)
  :=
  \sum_{e=1}^E r_e
  \le
  2\Bigl(\sum_{t=0}^{H-1}\gamma^t\Bigr)
  \sum_{e=1}^E \beta_e.
\)
In particular, if \(\beta_e \le B e^{-1/2}\), then
\[
  \frac{\mathrm{Reg}(E)}{E}
  =
  O(E^{-1/2}).
\]
\end{theorem}

This bound makes explicit that the market need not be globally
centralized: as long as bids become calibrated to specialist value, the
auction tracks the oracle coordinator.

\subsection{Bucket-brigade payments recover structured credit assignment}
\label{sec.bucket_brigade_recover_credit_ass}

Bucket-brigade payments also provide a mechanism for assigning credit to
earlier agents. The next two results show that, under additional
structure, the payment received from the next winner can be interpreted
as either a Bellman continuation value or an ordered marginal contribution
in a workflow.

\begin{proposition}[Bellman-like credit under continuation pricing]
\label{prop:bellman-like}
If the expected payment received from the next winner equals the
discounted continuation value,
\[
  \mathbb{E}[Y_t \mid h_t,a_t^\star]
  =
  \gamma
  \mathbb{E}[V^{\mathrm{bb}}(h_{t+1})\mid h_t,a_t^\star],
\]
then the bucket-brigade target satisfies
\[
  \mathbb{E}[r_t+Y_t\mid h_t,a_t^\star]
  =
  \mathbb{E}[r_t+\gamma V^{\mathrm{bb}}(h_{t+1})
  \mid h_t,a_t^\star].
\]
Thus the winner's wealth update uses a Bellman-like one-step target.
\end{proposition}

\begin{theorem}[Shapley-like credit in DAG workflows]
\label{thm:shapley-like}
Consider an acyclic workflow in which subtasks are executed in
topological order, and suppose the next winner's bid prices the
continuation value of the remaining feasible subgraph. Then the
bucket-brigade payment to the current subtask agent coincides, up to a
downstream baseline, with that agent's ordered marginal contribution to
the final outcome value.
\end{theorem}

Hence bucket-brigade transfers do more than redistribute wealth: they can
serve as a local credit-assignment signal that propagates downstream
value backward through the executed chain of agents.


Throughout the proofs, we write \(x\) for a decision context. In a fully
observed Markov environment, \(x=o_t=s_t\); in partially observed
settings, \(x\) may denote the history
\[
  h_t=(o_0,a_0^{\mathrm{env}},r_0,\ldots,o_t).
\]
For an agent \(a\), let \(b_a\) denote its fixed bid and \(W_a\) its
wealth. When \(a\) wins at context \(x\), define its resale-adjusted
economic payoff as
\[
  G_a(x)
  :=
  r_t + \mathbf{1}\{t+1 < H\} b_{a_{t+1}^\star},
\]
where \(r_t\) is the immediate environment reward and
\(b_{a_{t+1}^\star}\) is the next winning bid paid back to \(a\) under
the bucket-brigade transfer rule. At a terminal step, the resale term is
zero. Thus the net wealth increment of a winning agent, ignoring rent, is
\[
  G_a(x)-b_a.
\]
Let
\[
  V(a,x):=\mathbb{E}[G_a(x)],
  \qquad
  V^\star(x):=\sup_{a:\phi_a(x)=1} V(a,x).
\]
Let \(\mathsf{Surv}_\infty(x)\) denote the set of agents that win at
context \(x\) infinitely often and are never removed, and define the
limiting surviving bid frontier
\[
  \beta_\infty(x)
  :=
  \sup_{a\in\mathsf{Surv}_\infty(x)} b_a,
\]
with \(\sup\varnothing:=0\).

\subsection{Malicious agents and collusion}
\label{sec:malicious_agents}

Finally, the bankruptcy rule discourages agents that bid aggressively but
do not create commensurate value. A malicious agent whose bid exceeds its
expected resale-adjusted payoff has negative wealth drift and is removed
in finite time.

For colluding agents, the relevant object is not individual wealth but
coalition wealth, because internal transfers cancel. A cartel can survive
only if some surviving sub-coalition has nonnegative aggregate external
wealth drift.

Thus collusion changes the unit of selection from an individual agent to
a surviving cartel. Bankruptcy eliminates colluding agents only when
every possible surviving cartel is externally loss-making; otherwise a
cartel with nonnegative aggregate drift may persist or even form a
monopoly.

\subsection{Proofs for \Cref{sec:theory-selection}}

\begin{lemma}[Survival of underpriced agents]
\label{lem:survival}
Suppose context \(x\) is recurrent and the payoff sequence of agent \(a\)
at \(x\) is stationary and bounded. If
\[
  V(a,x)-b_a \ge \delta > 0,
\]
then agent \(a\) survives forever after its first win at \(x\) with
positive probability.
\end{lemma}

\begin{proof}[Proof of Lemma~\ref{lem:survival}]
Fix an agent \(a\) and context \(x\). Let \(G_{a,k}(x)\) be the
resale-adjusted payoff obtained on the \(k\)-th win of agent \(a\) at
\(x\). Define
\[
  X_k := b_a - G_{a,k}(x).
\]
By stationarity and boundedness, \((X_k)_{k\ge 1}\) is i.i.d. and
bounded. Moreover,
\[
  \mathbb{E}[X_k]
  =
  b_a - V(a,x)
  \le
  -\delta <0.
\]
Since \(X_k\) is bounded and has strictly negative mean, there exists
\(\lambda>0\) small enough such that
\[
  m_\lambda
  :=
  \mathbb{E}\!\left[e^{\lambda X_k}\right]
  <1.
\]
Then
\[
  M_n
  :=
  \exp\!\left(\lambda\sum_{k=1}^n X_k\right)
\]
is a nonnegative supermartingale.

Let \(T\) be the bankruptcy time measured in wins of \(a\) at context
\(x\):
\[
  T
  :=
  \inf\left\{
    n\ge 1:
    \sum_{k=1}^n X_k > W_a(0)
  \right\},
\]
where \(W_a(0)\) is the wealth immediately before its first win at \(x\).
On \(\{T\le n\}\),
\[
  M_T
  \ge
  e^{\lambda W_a(0)}.
\]
By optional stopping applied to \(T\wedge n\),
\[
  1=M_0
  \ge
  \mathbb{E}[M_{T\wedge n}]
  \ge
  e^{\lambda W_a(0)}\mathbb{P}(T\le n).
\]
Thus
\[
  \mathbb{P}(T\le n)
  \le
  e^{-\lambda W_a(0)}.
\]
Letting \(n\to\infty\) gives
\[
  \mathbb{P}(T<\infty)
  \le
  e^{-\lambda W_a(0)}<1.
\]
Therefore
\[
  \mathbb{P}(T=\infty)
  \ge
  1-e^{-\lambda W_a(0)}
  =:q_\delta>0.
\]
\end{proof}

\begin{proof}[Proof of Theorem~\ref{thm:selection}]
We prove the upper and lower bounds separately.

\paragraph{Upper bound.}
Fix any \(a\in\mathsf{Surv}_\infty(x)\). After its first \(n\) wins at
\(x\), its wealth is bounded above by
\[
  W_a(n)
  \le
  W_a(0)+\sum_{k=1}^n \bigl(G_{a,k}(x)-b_a\bigr),
\]
ignoring nonpositive rent terms. By the strong law of large numbers,
\[
  \frac{1}{n}
  \sum_{k=1}^n \bigl(G_{a,k}(x)-b_a\bigr)
  \xrightarrow{\mathrm{a.s.}}
  V(a,x)-b_a.
\]
If \(b_a>V(a,x)\), then \(W_a(n)\to -\infty\) almost surely, so \(a\)
would eventually be removed. This contradicts
\(a\in\mathsf{Surv}_\infty(x)\). Hence every forever-surviving recurrent
winner satisfies
\[
  b_a\le V(a,x)\le V^\star(x).
\]
Taking the supremum gives
\[
  \beta_\infty(x)\le V^\star(x).
\]

\paragraph{Lower bound.}
Fix \(\delta>0\). Suppose, toward contradiction, that with positive
probability
\[
  \beta_\infty(x)
  <
  V^\star(x)-\varepsilon_{\max}-2\delta.
\]
On this event, the recurrent solvent frontier at \(x\) is eventually
below \(V^\star(x)-\varepsilon_{\max}-2\delta\). Since \(x\) is recurrent,
there are infinitely many later opportunities for injection. By the
entry condition, with probability at least \(\lambda_\delta>0\), a new
eligible agent \(a'\) is injected with
\[
  V(a',x)\ge V^\star(x)-\delta.
\]
Under the novice rule,
\[
  b_{a'}
  =
  \max_{a\in C_t} b_a+\varepsilon_{a'},
  \qquad
  \varepsilon_{a'}\in(0,\varepsilon_{\max}],
\]
where \(C_t\) is the competing eligible set when \(a'\) first competes.
Since the competing frontier is below
\(V^\star(x)-\varepsilon_{\max}-2\delta\), we have
\[
  b_{a'}
  \le
  V^\star(x)-2\delta.
\]
Therefore
\[
  V(a',x)-b_{a'}
  \ge
  \delta.
\]
By Lemma~\ref{lem:survival}, each such entrant survives forever after its
first win at \(x\) with probability at least \(q_\delta>0\). Across
infinitely many independent injection opportunities, the probability
that no such profitable entrant survives is zero. Hence, almost surely,
one such entrant eventually survives, contradicting the assumed upper
bound on \(\beta_\infty(x)\). Therefore
\[
  \beta_\infty(x)
  \ge
  V^\star(x)-\varepsilon_{\max}-2\delta.
\]
Letting \(\delta\downarrow 0\) gives
\[
  \beta_\infty(x)
  \ge
  V^\star(x)-\varepsilon_{\max}.
\]
Combining the two bounds proves the theorem.
\end{proof}

\subsection{Proofs for \Cref{sec:outcome-suffices}}

\begin{proof}[Proof of Theorem~\ref{thm:outcome-suffices}]
Let
\[
  V_t^{\mathrm{auc}}(h_t)
  :=
  \mathbb{E}_{\pi^{\mathrm{auc}}}\!\left[
    \sum_{u=t}^{H-1}\gamma^{u-t}r_u^{\mathrm{out}}
    \,\middle|\, h_t
  \right],
\]
and define
\[
  \Delta_t(h_t)
  :=
  V_t^\star(h_t)-V_t^{\mathrm{auc}}(h_t).
\]
At any reachable history with \(E_t\neq\varnothing\), the assumption gives
\[
  Q_t^\star(h_t,a_t^\star)
  \ge
  \max_{u\in E_t} Q_t^\star(h_t,u)-\varepsilon.
\]
After the auction winner acts, the remaining loss is the future auction
suboptimality. Therefore
\[
  \Delta_t(h_t)
  \le
  \varepsilon
  +
  \gamma
  \mathbb{E}\!\left[
    \Delta_{t+1}(h_{t+1})
    \,\middle|\,
    h_t,a_t^\star
  \right].
\]
Since \(\Delta_H(h_H)=0\), backward induction yields
\[
  \Delta_0(h_0)
  \le
  \varepsilon\sum_{t=0}^{H-1}\gamma^t.
\]
Taking expectation over the initial history proves the theorem.
\end{proof}

\subsection{Proofs for \Cref{sec:oracle-regret}}

\begin{proof}[Proof of Theorem~\ref{thm:oracle-regret-decay}]
Fix episode \(e\) and reachable history \(h_t\). Let
\(u^{\mathrm{orc}}_{e,t}\) be the oracle-selected eligible agent, and let
\(u^{\mathrm{auc}}_{e,t}\) be the auction winner. By bid concentration,
\[
  Q^\star_{e,t}(h_t,u^{\mathrm{orc}}_{e,t})
  \le
  b_{e,t}(h_t,u^{\mathrm{orc}}_{e,t})+\beta_e.
\]
Because the auction maximizes bids,
\[
  b_{e,t}(h_t,u^{\mathrm{orc}}_{e,t})
  \le
  b_{e,t}(h_t,u^{\mathrm{auc}}_{e,t}).
\]
Applying bid concentration again,
\[
  b_{e,t}(h_t,u^{\mathrm{auc}}_{e,t})
  \le
  Q^\star_{e,t}(h_t,u^{\mathrm{auc}}_{e,t})+\beta_e.
\]
Combining the three inequalities gives
\[
  Q^\star_{e,t}(h_t,u^{\mathrm{auc}}_{e,t})
  \ge
  \max_{u\in E_{e,t}(h_t)}
  Q^\star_{e,t}(h_t,u)
  -
  2\beta_e.
\]
Thus, in episode \(e\), the auction winner is \(2\beta_e\)-optimal at
every reachable history. Applying
\Cref{thm:outcome-suffices} with \(\varepsilon=2\beta_e\) gives
\[
  r_e
  \le
  2\beta_e\sum_{t=0}^{H-1}\gamma^t.
\]
Summing over \(e=1,\ldots,E\) proves the cumulative regret bound. If
\(\beta_e\le B e^{-1/2}\), then
\[
  \sum_{e=1}^E \beta_e
  \le
  B\sum_{e=1}^E e^{-1/2}
  \le
  2B\sqrt E,
\]
which gives
\[
  \mathrm{Reg}(E)
  \le
  4B
  \Bigl(\sum_{t=0}^{H-1}\gamma^t\Bigr)
  \sqrt E.
\]
Dividing by \(E\) yields the average-regret rate.
\end{proof}

\subsection{Proofs for \Cref{sec.bucket_brigade_recover_credit_ass}}

\begin{proof}[Proof of Proposition~\ref{prop:bellman-like}]
By assumption,
\[
  \mathbb{E}[Y_t\mid h_t,a_t^\star]
  =
  \gamma
  \mathbb{E}[V^{\mathrm{bb}}(h_{t+1})\mid h_t,a_t^\star].
\]
Substituting this identity into
\[
  \Delta W^{\mathrm{win}}_{a_t^\star,t}
  =
  r_t+Y_t-b_{a_t^\star}
\]
and taking conditional expectations gives
\[
  \mathbb{E}[\Delta W^{\mathrm{win}}_{a_t^\star,t}\mid h_t]
  =
  \mathbb{E}[r_t+\gamma V^{\mathrm{bb}}(h_{t+1})
  \mid h_t,a_t^\star]
  -
  b_{a_t^\star}.
\]
\end{proof}

\begin{proof}[Proof sketch of Theorem~\ref{thm:shapley-like}]
Because the workflow is acyclic, any execution induces a topological
order. Completing subtask \(u\) enlarges the completed prefix from \(S\)
to \(S\cup\{u\}\). Under the continuation-pricing assumption, the
successor's willingness to pay equals the continuation value of the
expanded feasible subgraph, up to a downstream baseline. Therefore the
payment flowing back to the current agent measures the additional
downstream value unlocked by executing \(u\). This is precisely the
ordered marginal contribution
\[
  v(S\cup\{u\})-v(S),
\]
up to the baseline term. Summing over the realized topological order
yields a telescoping decomposition of the final outcome value.
\end{proof}

\subsection{Proofs for \Cref{sec:malicious_agents}}

\begin{lemma}[Extinction of malicious agents]
\label{lem:malicious}
Suppose payoffs are recurrent, stationary, and bounded, and suppose rent
\(\rho>0\) is charged infinitely often. If an agent \(a\) is malicious in
the sense that there exists \(\eta>0\) such that
\[
  V(a,x)\le b_a-\eta
\]
at every recurrent context \(x\) where \(a\) is eligible, then \(a\) is
removed in finite time almost surely.
\end{lemma}

\begin{proof}[Proof of Lemma~\ref{lem:malicious}]
Fix an agent \(a\). Suppose first that \(a\) wins infinitely often. Let
\(G_{a,k}\) be the resale-adjusted payoff on its \(k\)-th win. Its wealth
after \(n\) wins satisfies
\[
  W_a(n)
  =
  W_a(0)
  +
  \sum_{k=1}^{n}\bigl(G_{a,k}-b_a\bigr)
  -
  \rho N_a^{\mathrm{ep}}(n),
\]
where \(N_a^{\mathrm{ep}}(n)\) is the number of elapsed rent-charging
episodes by the time of its \(n\)-th win. For a malicious agent,
\[
  \mathbb{E}[G_{a,k}-b_a]\le -\eta.
\]
By the strong law of large numbers,
\[
  \frac{1}{n}
  \sum_{k=1}^{n}\bigl(G_{a,k}-b_a\bigr)
  \to
  \mathbb{E}[G_{a,1}-b_a]
  \le
  -\eta
  \qquad\text{a.s.}
\]
The rent term is nonpositive, so \(W_a(n)\to -\infty\) almost surely.
Thus bankruptcy occurs in finite time.

If \(a\) wins only finitely many times, then after its final win it no
longer receives downstream payments. Since \(\rho>0\) and rent is charged
infinitely often, rent alone eventually drives \(W_a\) below zero. Hence
\(a\) is removed in finite time almost surely.
\end{proof}

Before proving the collusion result, we record the coalition wealth
decomposition. For a nonempty set of agents \(S\), define
\[
  W_S(t):=\sum_{a\in S\cap\mathcal A_t}W_a(t).
\]
Bucket-brigade transfers between two agents in \(S\) cancel in
\(W_S(t)\). Therefore, over an episode segment \(\tau\), the only changes
to \(W_S\) come from environment rewards, payments entering \(S\), and
payments leaving \(S\). If \(p_t\) denotes the predecessor of the winner
at step \(t\), with \(p_t=\varnothing\) at the first winning step of an
episode, then
\[
\begin{aligned}
  \Delta W_S(\tau)
  ={}& \sum_{t\in\tau}\mathbf 1\{a_t^\star\in S\}\, r_t \\
     &+ \sum_{t\in\tau}\mathbf 1\{a_t^\star\notin S,\ p_t\in S\}\, b_{a_t^\star} \\
     &- \sum_{t\in\tau}\mathbf 1\{a_t^\star\in S,\ p_t\notin S\cup\{\varnothing\}\}\, b_{a_t^\star} \\
     &- \sum_{t\in\tau}\mathbf 1\{a_t^\star\in S,\ p_t=\varnothing\}\, b_{a_t^\star}.
\end{aligned}
\]

\begin{lemma}[Extinction of externally loss-making collusions]
\label{lem:colluding-malicious}
Let \(S\) be a coalition of agents that may coordinate bids, prompts, and
internal transfers. If every nonempty surviving sub-coalition
\(T\subseteq S\) has strictly negative expected aggregate external wealth
drift on recurrent episode segments, then all agents in \(S\) are removed
in finite time almost surely.
\end{lemma}

\begin{proof}[Proof of Lemma~\ref{lem:colluding-malicious}]
Consider the currently surviving colluding set \(T\subseteq S\). By
assumption, whenever \(T\) is active on recurrent episode segments,
\[
  \mathbb E[\Delta W_T(\tau)\mid \mathcal F_{\tau^-}]
  \le
  -\eta_T
\]
for some \(\eta_T>0\). The increments are bounded, so the coalition
wealth \(W_T\) has strictly negative drift on recurrent segments. By the
strong law of large numbers for the corresponding bounded adapted
increments, \(W_T\) eventually falls below zero unless some member of
\(T\) is removed earlier. Once a member is removed, the argument is
restarted with the smaller surviving sub-coalition. The assumption holds
for every nonempty sub-coalition that can remain active after
bankruptcies, so this induction continues until no member of \(S\)
survives. Hence all agents in \(S\) are removed in finite time almost
surely.
\end{proof}

\section{Extended Experiment Details}
\label{sec:exp-details}

This appendix provides task-level and baseline-level details omitted from the main experiment section. In the main text, we use the term \emph{partial agent} for any agent whose capability is incomplete relative to the full task interface. Partiality can arise from explicit constraints, such as limited tools or output length, or from specialization, such as role-specific prompting. This section specifies how partiality is instantiated in each domain.

We compare \ours{} with complete-agent, partial-agent, and domain-specific baselines. The experiments were conducted primarily for inference-time evaluation. We use NVIDIA H200 GPUs for running local models, with at most a single H200 GPU required per experiment. In addition to local computation, we leveraged several commercial APIs, including the official APIs from OpenAI, Google Gemini, and Anthropic Claude, to evaluate model performance across different providers.

\subsection{Baseline Details}
\label{sec:exp-baselines}

\paragraph{Complete-agent baselines.}
\textsc{ReAct}~\citep{yao2022react} is the primary complete-agent baseline. In each domain where it is used, a single agent with the same backbone model is allowed to solve the task end-to-end with access to the full action or tool space. This baseline tests whether economic organization among partial agents can compensate for capability that is concentrated in a single complete agent.

\textsc{GEA}~\citep{weng2026group} is also instantiated as a complete-agent baseline. Like \textsc{ReAct}, it uses complete agents with access to the full task interface, but additionally provides a self-improvement mechanism based on experience sharing and agent evolution. This baseline tests whether the gains of \ours{} come specifically from market-based execution-time coordination and wealth-driven population evolution, rather than from self-improvement alone.

\textsc{OpenEvolve}~\citep{openevolve,novikov2025alphaevolve} is used for distributed-system optimization. It performs iterative program improvement with a monolithic evolutionary coding agent. In Cloudcast, \textsc{OpenEvolve} follows its standard setup with 300 iterations, while \ours{} runs for 30 episodes.

\paragraph{Partial-agent baselines.}
For partial-agent comparison, we use Multi-Agent Debate~\citep{du2024improving}. This baseline uses multiple agents that interact with one another, but it does not include auctions, bid-based control allocation, peer-to-peer payments, rent, bankruptcy, or wealth-based exploration and exploitation. It isolates the effect of economic coordination from the effect of merely using multiple partial agents.

We also compare against a best-of-$N$ multi-agent baseline on Cloudcast, where $N$ is set to the number of episodes used by \ours{}. This baseline tests whether improvements can be explained by repeated multi-agent sampling without market-driven evolution.

\paragraph{Domain-specific baselines.}
For accelerator design, we compare against \textsc{DOSA}~\citep{hong2023dosa}, a strong non-LLM method for differentiable model-based one-loop search in DNN accelerator optimization.

\subsection{Mathematical Reasoning}
\label{sec:exp-math}

We evaluate on MATH~\citep{hendrycks2021measuring}, a competition-level mathematics benchmark whose problems are annotated from Level~1 to Level~5 by difficulty. We randomly sample 20 training problems from each difficulty level and train on an easy-to-hard stream from Level~1 to Level~5. Evaluation reports greedy pass@1 accuracy separately by difficulty level.

The population is initialized with a planner--executor--verifier decomposition. The planner proposes the immediate next step, the executor carries out the proposed step, and the verifier judges whether the executed step is correct. These agents are partial in two ways. First, each agent is role-specific and therefore sees the task through a limited functional interface. Second, each agent is given a short maximum output budget of 64--256 tokens, preventing any individual agent from solving the entire problem with a long chain-of-thought style response.

We intentionally use relatively weak LLM backbones, Llama-3.1-8B and Gemma-2-9B, to test whether economic organization can amplify partial individual capabilities. The complete-agent baseline is a single end-to-end agent with the corresponding backbone and full output/action capability.

\subsection{Financial Research}
\label{sec:exp-finance}

We evaluate on Finance-Agent-Bench~\citep{bigeard2025finance}, a benchmark of real-world financial research tasks over company filings. The environment provides four tools. We randomly select 30 tasks for training and use the remaining 20 for testing.

Each partial agent is restricted to exactly one tool, so no individual partial agent can complete the full research task alone. The population therefore must coordinate across tool-specialized agents, including agents responsible for filing search, web search, HTML parsing, retrieval, and answer submission. We compare against Multi-Agent Debate under the same partial-agent setting, as well as complete-agent \textsc{ReAct} and \textsc{GEA} baselines that can access all tools.

For the generalist-v.s.-specialist study in \Cref{fig:mechanism-robustness}, we introduce a complete generalist agent with access to all four tools, alongside the ordinary partial specialist agents. Appendix~\ref{sec:case-study-generalist-no-monopoly} analyzes why this generalist does not monopolize the market.

\subsection{Scientific Research}
\label{sec:exp-science}

We evaluate on the \emph{Research} subset of FrontierScience~\citep{wang2026frontierscience}, which assesses open-ended scientific problem-solving. We adopt a 40/20 train--test split.

The society is initialized with four role-specialized partial agents: literature, planner, executor, and verifier. The literature agent lists possible background knowledge required for the scientific question. The planner decomposes the problem into subgoals. The executor advances one unresolved sub-question with algebraic or scientific reasoning. The verifier checks whether intermediate reasoning is valid and flags errors. These agents are partial because each is prompted to operate only within one stage of the problem-solving pipeline.

We use Gemini-3-Flash as the backbone model and compare against \textsc{GEA} under the same backbone. We report both mean accuracy and best-run accuracy across evaluated checkpoints during the evolution process. Appendix~\ref{sec:case-study-prompt-evolution} and Appendix~\ref{sec:topology-evolution-research} provide case studies showing how scientific reasoning prompts and auction-selected topologies evolve over training.

\subsection{Accelerator Design}
\label{sec:exp-dse}

We evaluate EOM on accelerator design, following the task description in Section 3.1. The task is a discrete search problem over hardware mappings for deep learning accelerators. We use the standardized GEMMINI benchmark suite~\cite{genc2021gemmini}, which provides a hardware simulation environment for 24 ResNet-50 convolution kernels. For each kernel, the goal is to find a hardware mapping that minimizes energy-delay product (EDP), with lower EDP indicating a better design. Broader context on generative models applied across the systems stack, including DSE, can be found in~\cite{prakash2025quarch, tschand2026genaisystemsrecurringchallenges,sun2022gibbon}.

The population contains three role-specialized partial agent types. \emph{Historians} summarize previous trajectory trials and maintain memory of promising or failed design directions. \emph{Planners} propose architectural or mapping-level search directions. \emph{Executors} carry out fine-grained local evaluations. These roles are partial because each agent is responsible for only one component of the long-horizon design-space exploration loop.

We compare against the non-LLM \textsc{DOSA} baseline~\citep{hong2023dosa} and a complete \textsc{ReAct} agent. Both LLM-based methods use the Gemma-4-31B backbone. The main metric is average EDP across the 24 kernels, with lower values indicating better designs. \Cref{fig:dse_dynamics} analyzes wealth dynamics during search, and \Cref{fig:mechanism-robustness} reports per-kernel EDP improvements.

\begin{table}[t]
  \centering
  \small
\caption{\textbf{Representative discovered mappings.} Example accelerator mappings discovered by \textsc{EoM} on representative kernels, with EDP reduction reported \emph{relative to \textsc{DOSA}}. The learned solutions repeatedly recover effective output-stationary motifs. \textbf{Notation:} \texttt{L0}--\texttt{L3} index memory-hierarchy levels (\texttt{L0}=innermost registers, \texttt{L3}=outermost DRAM); \texttt{K},\texttt{C} are output- and input-channel loops; \texttt{P},\texttt{Q} are output spatial dimensions; \texttt{R},\texttt{S} are filter spatial dimensions; \texttt{N} is the batch loop; suffix \texttt{X} marks a spatially parallelized loop; \texttt{[O]} marks the output-stationary level.}

  \label{tab:dse_motifs}
  \begin{tabular}{l l c}
  \toprule
  \textbf{Kernel} & \textbf{\ours{} best mapping} & \textbf{EDP reduction} \\
  \midrule
  Kernel 1 & \texttt{L3 K8 - L2 C2 K128X - \textbf{L1[O] P14 Q14 C128X} - L0 N1} & $36.5\times$ \\
  Kernel 2 & \texttt{L3 K2 - L2 C8 K128X - \textbf{L1[O] P14 Q14 C128X} - L0 N1} & $23.4\times$ \\
  Kernel 3 & \texttt{L3 K2 R3 S3 - L2 C2 K128X - \textbf{L1[O] P14 Q14 C128X} - L0 N1} & $11.9\times$ \\
  Kernel 4 & \texttt{L3 K32 - L2 C8 K64X - \textbf{L1[O] P7 Q7 C64X} - L0 N1} & $8.7\times$ \\
  Kernel 5 & \texttt{L3 R3 S3 - L2 K128X - \textbf{L1[O] P28 Q28 C128X} - L0 N1} & $4.9\times$ \\
  Kernel 6 & \texttt{L3 K8 P2 Q2 - L2 C2 K32X - \textbf{L1[O] P28 Q28 C32X} - L0 N1} & $3.1\times$ \\
  Kernel 7 & \texttt{L3 K4 R3 S3 - L2 C4 K128X - \textbf{L1[O] P7 Q7 C128X} - L0 N1} & $2.2\times$ \\
  Kernel 8 & \texttt{L3 K2 R3 S3 - L2 C2 K128X - \textbf{L1[O] P14 Q14 C128X} - L0 N1} & $0.4\times$ \\
  \bottomrule
  \end{tabular}
\end{table}

\subsection{Distributed-System Optimization}
\label{sec:exp-cloudcast}

We evaluate on the Cloudcast task from the ADRS benchmark~\citep{cheng2025barbarians}, which minimizes total data-transfer cost in a multi-region, multi-cloud environment. We formulate Cloudcast as iterative code optimization. Each episode starts from the last successfully verified program, regardless of its score, and receives a reward at the end. This resembles test-time reinforcement learning~\citep{wang2025thetaevolve,ttt-discover2026}, where improvement attempts are made from a parent solution and updated using reward signals.

The population uses role-specialized partial agents commonly used in coding tasks. The Planner sets subgoals; the Reader inspects the codebase and summarizes findings; the Implementer edits code; the Builder compiles and tests; the Evaluator calls the verifier and records score changes; and the Finalizer submits the run. These agents are partial because each is responsible for only one stage of the coding and verification workflow.

We compare \ours{} with \textsc{OpenEvolve}~\citep{openevolve,novikov2025alphaevolve} and a best-of-$N$ multi-agent baseline. All methods use GPT-5-mini. \ours{} runs for 30 episodes, while \textsc{OpenEvolve} uses its standard 300-iteration setup, which takes longer in wall-clock time. Appendix~\ref{sec:case-study-cloudcast} traces how prompt-level action discipline and auction-selected topology evolve over a representative Cloudcast run.

\subsection{Evaluation Protocol}
\label{sec:exp-eval-protocol}

During training, \ours{} alternates between within-episode planning and across-episode population adaptation. Within an episode, agents wake up according to their local predicates, compete via auctions, execute actions, and exchange payments through the bucket-brigade transaction rule. Across episodes, agents pay rent, agents with negative wealth are removed, and new agents are injected through exploration and exploitation.

During evaluation, the trained population is frozen. Bids, prompts, wealth values, and agent status are fixed; no payments, rewards, rent, births, or bankruptcies are applied. Each test task is evaluated using a thread-local copy of the trained population, and the auction mechanism is used only for action selection. Thus, reported test results measure the learned society at a fixed checkpoint rather than continued adaptation during testing.

\section{Case Study: Prompt and Topology Evolution in Scientific Research}
\label{sec:case-study-scientific-research}

This appendix analyzes the scientific research run at two coupled levels. At the micro level, economic selection evolves reusable reasoning strategies inside agents' prompts. At the macro level, these local strategies reshape the auction-selected collaboration topology. The central observation is that topology evolution is not independent of prompt evolution: once an \textsc{Executer} internalizes checks that previously required separate \textsc{Verifier} or \textsc{Literature} intervention, the wakeup landscape changes, and the same auction rules select a different execution path.

We focus on FrontierScience-Research, where the \textsc{Executer} as an example, advances one unresolved sub-question with explicit algebra or scientific reasoning. Across the 40-episode training run, nine of eleven successful episodes are carried by descendants of a single evolving \textsc{Executer} agent. Since each child inherits its parent's trainable prompt and receives only a small mutation, this agent family provides a traceable view of prompt-level evolution. The evidence is mechanistic rather than a controlled single-sentence intervention: we do not claim that a specific prompt line deterministically causes a specific answer. Instead, we show that economic selection accumulates reusable, falsifiable reasoning operators, and that these operators change both local behavior and system-level routing.

\subsection{Prompt-Level Evolution: From Generic Execution to Reusable Scientific Reasoning}
\label{sec:case-study-prompt-evolution}

The initial prompt is a generic execution policy. It asks the agent to complete the earliest unfinished sub-question, show intermediate algebra, track signs and dimensions, and repair verifier-flagged errors. This is a reasonable starting point, but it does not yet specify how to expose abstract relations as auditable scalar systems, how to check whether a problem is well-posed, or how to falsify intermediate results against the original equations.

\begin{lstlisting}[style=prompt,label={lst:executer-init}]
Identify the earliest sub-question in the plan that is not yet completed,
then carry out its derivation cleanly:
- State the sub-part label (e.g. 'Part (b):') at the top of your <step>.
- Show the intermediate algebra or physical reasoning; keep equations
  compact but include justification for non-trivial substitutions
  (limits, approximations, linearisations).
- Track signs, unit/dimension consistency, and small-parameter regimes.
- If a prior Executer turn has an error flagged by a Verifier, fix only
  that error in this turn instead of starting a new sub-part.
\end{lstlisting}

Table~\ref{tab:executer-motifs} summarizes the main prompt-level changes. The prompt does not merely become longer. It becomes more operational: each added motif specifies a concrete check or decomposition procedure that can be reused on future tasks.
\begin{table}[t]
\centering
\footnotesize
\setlength{\tabcolsep}{4pt}
\renewcommand{\arraystretch}{1.12}
\caption{Reusable reasoning motifs learned by the evolving \textsc{Executer} agent. Each mutation
adds a falsifiable operation rather than a task-specific fact.}
\begin{tabularx}{\linewidth}{@{}lYY@{}}
\toprule
Stage & Learned reasoning motif and local effect & Representative transfer \\
\midrule
Gen.\ 0
& Generic derivation with unit/sign checks; performs local algebra with limited self-auditing.
& Initial prompt \\

Gen.\ 1
& Component-wise expansion; converts abstract relations into auditable scalar systems.
& Josephson junction; dark-matter detection \\

Gen.\ 2--3
& Governing-principle identification, limiting cases, and constraints; anchors derivations and checks feasibility.
& Gravimetry; pharmacology \\

Gen.\ 4
& Equation-vs-DOF counting; detects underdetermined or overclaimed solutions.
& Protein purification; constrained inference \\

Gen.\ 5
& Symmetry scouting and substitute-back checks; internalizes verifier-like falsification.
& $^{195}$Pt NMR; biology; synthesis \\
\bottomrule
\end{tabularx}
\label{tab:executer-motifs}
\end{table}

The first major mutation replaces generic algebraic exposition with coordinate-based and component-wise execution. This change is distilled from a CMB trajectory in which treating $C_\ell^{TT}$, $C_\ell^{EE}$, and $C_\ell^{TE}$ as separate scalar relations proved more reliable than manipulating a single high-level covariance object. The resulting prompt requires the agent to expose all relevant variables and scalar equations, making intermediate reasoning easier for
\textsc{Verifier} agents to audit.

Subsequent mutations add increasingly explicit self-auditing. The agent learns to state the
governing principle before algebraic manipulation, test limiting cases or symmetries, verify boundary
and feasibility constraints, count independent equations against unknown degrees of freedom, and,
when possible, substitute the derived expression back into the governing equations. These edits turn
the \textsc{Executer} from a generic algebraic operator into a compact scientific reasoning module.

By generation 5, the prompt has grown from four generic bullets to eight operational checks:

\begin{lstlisting}[style=prompt,label={lst:executer-final}]
Identify the earliest sub-question in the plan that is not yet completed,
then carry out its derivation cleanly:
- State the sub-part label (e.g. 'Part (b):') and the core logic bridge
  or governing principle connecting the known variables to the target.
- Identify any symmetries, invariants, or simplified regimes that can
  be used to streamline the derivation or cross-check the result.
- Verify that the system is well-defined by comparing the number of
  independent equations to the degrees of freedom before proceeding
  with the math.
- Prioritize explicit coordinate-based or component-wise derivations;
  define all basis coefficients and variables clearly.
- Show full algebraic expansions and provide the complete system of
  scalar equations for all degrees of freedom.
- Track signs, unit/dimension consistency, and small-parameter regimes;
  sanity-check results via limiting cases, symmetries, or functional
  scaling.
- Confirm the final expression satisfies all relevant constraints,
  boundary conditions, and logical bounds; when feasible, verify by
  substituting the result back into the governing equations.
- If a prior Executer turn has an error flagged by a Verifier, fix only
  that error in this turn instead of starting a new sub-part.
\end{lstlisting}

This final prompt is not simply a more verbose instruction to ``execute carefully.'' It specifies a
reusable scientific procedure: identify the principle, check symmetries, verify well-posedness, expand explicitly, enforce constraints, and falsify the result by substitution.

\subsection{Cross-Domain Transfer of the Evolved Reasoning Routine}
\label{sec.case-study-cross-domain}

Table~\ref{tab:executer-transfer} shows that the evolved routine transfers across domains. The agent
is shaped by physics tasks, refined on chemistry and pharmacology, and later reused in spectroscopy,
biology, and synthesis.

\begin{table}[h]
\centering
\small
\caption{Successful scientific-research episodes associated with the evolving
\textsc{Executer} agent. The transferable object is not factual content, but a reusable reasoning
routine.}
\begin{tabular}{rllr}
\toprule
Ep. & Subject & Main reusable operation & Score \\
\midrule
8  & CMB parameter inference
   & Component-wise scalar decomposition & 0.70 \\
11 & Josephson junction
   & Auditable component-wise execution & 0.75 \\
14 & Dark-matter direct detection
   & Symmetry reduction and normalization checks & 0.80 \\
21 & NaCl/KCl gravimetry
   & Feasibility and constraint checking & 1.00 \\
25 & $\alpha4\beta2$ nAChR pharmacology
   & Subunit symmetry and counting & 0.70 \\
32 & $^{195}$Pt NMR spectroscopy
   & Governing-principle identification and substitution & 0.90 \\
33 & Protein purification
   & Equation-vs-DOF counting & 1.00 \\
36 & $\alpha7$ nAChR homopentamer
   & Stoichiometric and symmetry checks & 0.90 \\
37 & Pd-catalysed C--N bond
   & Constraint-aware synthesis reasoning & 0.80 \\
39 & $\alpha4\beta2$ nAChR stoichiometry
   & Reuse of subunit-counting routine & 0.80 \\
\bottomrule
\end{tabular}
\label{tab:executer-transfer}
\end{table}

The clearest transfer occurs in episode 32, the $^{195}$Pt NMR spectroscopy task. The prompt had
been shaped by CMB cosmology, dark-matter detection, electrolyte chemistry, and nAChR
pharmacology, none of which directly concerns hyperfine NMR shifts. Nevertheless, the inherited
routine applies directly: identify the governing principle, count equations versus unknowns, check
symmetry, expand the scalar relation, and substitute the result back into the observation. This
supports the interpretation that the economy evolves content-independent reasoning strategies rather
than memorized task solutions.

\subsection{From Local Reasoning Evolution to Topology Evolution}
\label{sec:topology-evolution-research}

The prompt-level changes above have a system-level consequence. As the \textsc{Executer} begins
to perform dimensional checks, equation-vs-DOF counting, constraint verification, and substitute-back
falsification internally, the marginal value of separate \textsc{Verifier} turns decreases in some states.
Similarly, once the \textsc{Executer} prompt begins by naming the governing principle, separate
\textsc{Literature} turns become less necessary for some tasks. Since wakeup decisions are local LLM
judgments over the current workspace and each agent's evolved prompt, these internalized routines
directly change who wakes up, who bids, and who acts.

We define an episode's topology as the ordered sequence of auction winners, represented at both the
role level and the agent-identity level. This topology is not a pre-specified execution graph. It is
reselected at every step through local wakeup decisions and bid competition.

\subsection{Early Regime: Explicit Multi-Role Auditing}

Early in training, successful trajectories often rely on explicit multi-role auditing. Episode 11, the
Josephson-junction task, is the longest successful trajectory in the run. It obtains a score of 0.75 with
ten auction steps and all five roles:
\[
\textsc{Lit.}
\rightarrow
\textsc{Plan}
\rightarrow
\textsc{Exe.}
\rightarrow
\textsc{Ver.}
\rightarrow
\textsc{Exe.}
\rightarrow
\textsc{Ver.}
\rightarrow
\textsc{Plan}
\rightarrow
\textsc{Exe.}
\rightarrow
\textsc{Ver.}
\rightarrow
\textsc{Ans.}
\]

This trajectory is long, but it is structured. Literature supplies background, Planner decomposes the
problem, Executer advances sub-parts, Verifier audits intermediate executions, Planner re-scopes
after the audits, and Answer synthesizes the final response. No agent designs this workflow globally.
It emerges from local wakeup judgments, novice bidding, same-role blocking, and the terminal answer
restriction.

This early topology compensates for immature local reasoning. Because the \textsc{Executer} has not
yet internalized enough verification, the society externalizes uncertainty through repeated
execute--verify loops.

\subsection{Late Regime: Compact Specialist Execution}

Later in training, the same auction mechanism can select a much shorter specialist workflow.
Episode 33, the protein-purification task, receives a score of 1.0 with only three steps:
\[
\textsc{Plan}
\rightarrow
\textsc{Exe.}
\rightarrow
\textsc{Ans.}
\]

This contraction is not caused by a smaller population. The population contains 14 agents, including
living \textsc{Literature} and \textsc{Verifier} agents. These agents still run their wakeup judges, but
return NO. The reason is not a change in the auction code; it is the evolved prompt state.
\texttt{Executer} already performs dimensional checks, equation-vs-DOF counting, constraint
verification, and substitute-back falsification. As a result, separate verifier intervention has lower
marginal value for this episode.

\subsection{Training-Wide Topology Trend}

Table~\ref{tab:research-topology-trace} summarizes the successful trajectories. Later episodes often
use compact 3--4 step paths despite larger populations. Long paths remain when the task still benefits
from explicit literature or verifier intervention, showing that topology is adaptive rather than simply
pruned.

\begin{table}[h]
\centering
\small
\caption{Auction-derived topology for successful scientific-research episodes. Compact paths
become common later in training, even though the population remains large.}
\begin{tabular}{r p{3.0cm} p{6.3cm} r r r}
\toprule
Ep. & Subject & Auction-derived path & Roles & Pop. & Score \\
\midrule
8
& CMB inference
& Plan $\to$ Exe. $\to$ Ans.
& 3 & 10 & 0.70 \\

11
& Josephson junction
& Lit. $\to$ Plan $\to$ Exe. $\to$ Ver. $\to$ Exe. $\to$ Ver. $\to$ Plan $\to$ Exe. $\to$ Ver. $\to$ Ans.
& 5 & 12 & 0.75 \\

14
& Dark-matter detection
& Lit. $\to$ Exe. $\to$ Ans.
& 3 & 12 & 0.80 \\

21
& NaCl/KCl gravimetry
& Exe. $\to$ Ver. $\to$ Exe. $\to$ Ver. $\to$ Ans.
& 3 & 17 & 1.00 \\

25
& $\alpha4\beta2$ nAChR
& Exe. $\to$ Ans. $\to$ Exe. $\to$ Ans.
& 2 & 15 & 0.70 \\

32
& $^{195}$Pt NMR
& Lit. $\to$ Plan $\to$ Exe. $\to$ Ans.
& 4 & 12 & 0.90 \\

33
& Protein purification
& Plan $\to$ Exe. $\to$ Ans.
& 3 & 14 & 1.00 \\

36
& $\alpha7$ nAChR homopentamer
& Lit. $\to$ Ver. $\to$ Lit. $\to$ Plan $\to$ Exe. $\to$ Ver. $\to$ Plan $\to$ Lit. $\to$ Exe. $\to$ Ans.
& 5 & 14 & 0.90 \\

37
& Pd-catalysed C--N bond
& Lit. $\to$ Plan $\to$ Exe. $\to$ Ans.
& 4 & 14 & 0.80 \\

39
& $\alpha4\beta2$ nAChR redux
& Lit. $\to$ Exe. $\to$ Ans.
& 3 & 16 & 0.80 \\
\bottomrule
\end{tabular}
\label{tab:research-topology-trace}
\end{table}

The key point is that shorter paths are not a consequence of having fewer available agents. The market
often has more agents to choose from, but selects fewer active roles because competence has
concentrated in load-bearing agents. Conversely, when residual uncertainty remains high, as in
episode 36, the auction still selects longer literature and verification chains. Thus, topology evolution
is conditional and task-sensitive.

\subsection{Population-Level Selection}

The population graph evolves in parallel with the execution graph. Births are not uniformly
distributed across roles. The \textsc{Executer} agent deepens through descendants because it repeatedly
appears on reward-bearing trajectories. The \textsc{Answer} agent family also broadens as final
synthesis becomes valuable once upstream reasoning becomes reliable. In contrast, \textsc{Literature}
and \textsc{Planner} remain comparatively shallow in this run.

This reflects the economic feedback loop. Agents that create downstream value accumulate wealth,
survive longer, and become more likely to seed future mutations. Therefore, the economy selects not
only who acts in the current episode, but also where future evolutionary capacity is allocated. The
central evidence of evolution is therefore not merely that later agents solve more tasks. It is that the
market converts repeated successes and failures into reusable reasoning routines, and these routines
reorganize the society's workflow without centralized orchestration.
\section{Case Study: Prompt and Topology Evolution in CloudCast}
\label{sec:case-study-cloudcast}

This appendix provides evidence that the same economic mechanism also works in a code-optimization
environment. In Scientific Research, the main learned object is a reusable reasoning routine. In
CloudCast, the main learned object is an action discipline: agents learn when \emph{not} to spend an
expensive action. The topology adapts accordingly. When the workspace is close to a new high score,
the auction selects a short read--edit--evaluate--commit path; when the workspace still contains
uncertain regressions or unresolved implementation choices, the auction selects longer edit--build--evaluate
loops. Thus, the economy does not simply shorten workflows. It selects a workflow whose shape
matches the current workspace state.

CloudCast asks the society to evolve a single Python file, \texttt{initial\_program.py}, implementing
a multi-cloud broadcast routing algorithm. Given source and destination regions and a network graph
with edge costs and throughputs, the program must produce a \texttt{BroadCastTopology}. The simulator
scores the program by total egress cost across five inter- and intra-cloud scenarios. The seed
single-path Dijkstra implementation costs approximately \(\$1035\), and the reported score is the
fractional reduction \(\max(0, 1-\mathrm{cost}/1035)\). The workspace persists across episodes, so each
episode continues from the previous program rather than starting from scratch.

\subsection{Auction-Derived Topology}

The CloudCast implementation contains six roles:
\textsc{Reader}, \textsc{Planner}, \textsc{Implementer}, \textsc{Builder}, \textsc{Evaluator}, and
\textsc{Finalizer}. The \textsc{Reader} inspects files, the \textsc{Planner} proposes sub-goals, the
\textsc{Implementer} edits code, the \textsc{Builder} runs a build or import check, the
\textsc{Evaluator} calls \texttt{request\_eval()}, and the \textsc{Finalizer} submits the program via
\texttt{final\_answer}. At each step, every living agent runs a wakeup judge; eligible agents submit
fixed bids, and the highest bidder acts. CloudCast does not use same-role blocking, so consecutive
same-role turns are emergent rather than forbidden. The only structural restriction is terminal:
only \textsc{Finalizer}-tagged agents may bid near the end.

Table~\ref{tab:cloudcast-topology-trace} summarizes the completed checkpoint episodes. Unlike
the scientific-research case, path length does not monotonically decrease. This is expected: because
the codebase persists, some episodes are one edit away from improvement, while others require
multi-edit search or regression repair.

\begin{table}[h]
\centering
\small
\setlength{\tabcolsep}{4pt}
\caption{Auction-derived topology for completed CloudCast checkpoint episodes. Superscripts
denote consecutive same-role turns. Short paths occur when the persisted workspace is close to a
new high; long paths occur during multi-edit or regression-repair phases.}
\resizebox{\textwidth}{!}{%
\begin{tabular}{rrlrr}
\toprule
Ep. & Score & Auction-derived topology & Steps & Roles \\
\midrule
0   & 0.015 & R $\to$ P $\to$ I $\to$ E $\to$ F
& 5  & 5 \\
4   & 0.015 & R $\to$ I $\to$ B$^2$ $\to$ I$^4$ $\to$ B$^3$ $\to$ E $\to$ F
& 13 & 5 \\
9   & 0.107 & R $\to$ I $\to$ E $\to$ F
& 4  & 4 \\
12  & 0.107 & R $\to$ E $\to$ I $\to$ B $\to$ E $\to$ I $\to$ B $\to$ I $\to$ B $\to$ E$^2$ $\to$ F
& 12 & 4 \\
15  & 0.119 & P $\to$ R $\to$ E $\to$ F $\to$ I$^2$ $\to$ B$^2$ $\to$ E $\to$ I $\to$ B $\to$ E $\to$ I$^6$ $\to$ B $\to$ E $\to$ F$^2$
& 22 & 5 \\
20  & 0.154 & R $\to$ E $\to$ P $\to$ I $\to$ B$^2$ $\to$ E $\to$ F
& 8  & 5 \\
21  & 0.267 & P $\to$ R $\to$ I $\to$ B $\to$ E $\to$ I $\to$ E $\to$ F$^6$
& 13 & 5 \\
26  & 0.328 & P $\to$ R$^2$ $\to$ I$^3$ $\to$ B $\to$ E $\to$ F
& 9  & 5 \\
27  & 0.365 & R $\to$ I $\to$ B $\to$ E $\to$ I$^2$ $\to$ E $\to$ I $\to$ B $\to$ I $\to$ E $\to$ F
& 12 & 5 \\
28  & 0.365 & R $\to$ I $\to$ E $\to$ B $\to$ I $\to$ E $\to$ F$^3$
& 9  & 5 \\
\bottomrule
\end{tabular}%
}
\label{tab:cloudcast-topology-trace}
\end{table}

The table shows two patterns. First, topology length tracks residual uncertainty rather than training
time. Episode 9 is a clean one-edit improvement and collapses to four steps. Episode 15 is a
partition-sweep phase and stretches to the full step budget. Second, the role mix becomes more
selective. Late episodes often rely on a \textsc{Reader}/\textsc{Implementer}/\textsc{Evaluator}/
\textsc{Finalizer} core, while \textsc{Planner} and \textsc{Builder} appear when the workspace state
calls for them. The auction is unchanged; the local wakeup landscape changes as prompts and the
workspace evolve.

Episode 27 is the peak checkpoint, with score \(0.365\), or \(36.5\%\) cost reduction relative to the
Dijkstra seed (\$657 \text{ vs. } \$1035). Its topology is:
\[
\textsc{R}
\rightarrow
\textsc{I}
\rightarrow
\textsc{B}
\rightarrow
\textsc{E}
\rightarrow
\textsc{I$^2$}
\rightarrow
\textsc{E}
\rightarrow
\textsc{I}
\rightarrow
\textsc{B}
\rightarrow
\textsc{I}
\rightarrow
\textsc{E}
\rightarrow
\textsc{F}
\]
This path is not manually scripted. The \textsc{Reader} stops after summarizing the relevant file,
the \textsc{Implementer} continues while verifier feedback names unresolved scenarios, the
\textsc{Builder} wakes up when static validity is uncertain, the \textsc{Evaluator} wakes up when a
score-changing edit is plausible, and the \textsc{Finalizer} commits only when further iteration appears
unlikely to help.

\subsection{Prompt Evolution: Checks Before Costly Actions}

The topology changes because the agents' prompts change. Across the run, mutations occur in four
roles: \textsc{Evaluator}, \textsc{Builder}, \textsc{Implementer}, and \textsc{Finalizer}. The unifying
pattern is simple: each role learns a cheap structural check before its expensive action. The
\textsc{Evaluator} checks markers before \texttt{request\_eval()}, the \textsc{Builder} checks symbols
before running a build command, the \textsc{Implementer} states intent and re-reads after
\texttt{write\_file}, and the \textsc{Finalizer} rechecks invariants before \texttt{final\_answer}.

\begin{table}[h]
\centering
\small
\setlength{\tabcolsep}{4pt}
\renewcommand{\arraystretch}{1.15}
\caption{Prompt mutations in CloudCast. Each mutation adds an operational check at the point
where the previous failure surfaced. The learned content is not a new routing heuristic, but a better
discipline for deciding when to act.}
\begin{tabular}{rllp{4.2cm}p{3.6cm}}
\toprule
Gen. & Born & Mutated agent & New behavior & Trigger \\
\midrule
0 & ep.\ 0
& 6 founders
& ---
& Initial population \\

1 & ep.\ 2
& \texttt{Evaluator-9163}
& Cheap pre-eval checks before \texttt{request\_eval()}
& Eval budget spent on \(\Delta=0\) calls \\

2 & ep.\ 5
& \texttt{Evaluator-8866}
& Check markers, signatures, and required classes
& Refactor broke structural markers \\

3 & ep.\ 8
& \texttt{Builder-4001}
& Static symbol-presence check before build
& Broken-import builds wasted steps \\

4 & ep.\ 13
& \texttt{Builder-3586}
& Targeted test selection and actionable diagnostics
& Generic build checks missed semantic regressions \\

5 & ep.\ 19
& \texttt{Evaluator-5245}
& Per-scenario sign prediction before evaluation
& Aggregate \(\Delta \approx 0\) hid scenario regressions \\

6 & ep.\ 22
& \texttt{Finalizer-1195}
& Re-audit structural invariants before commit
& Broken invariant was committed \\

7 & ep.\ 25
& \texttt{Implementer-8828}
& State edit intent and re-read after writing
& Good behavior observed in recent thoughts \\

8 & ep.\ 26
& \texttt{Implementer-9766}
& Prefer smaller edits when uncertain
& Combined edits were hard to debug \\
\bottomrule
\end{tabular}
\label{tab:cloudcast-prompt-summary}
\end{table}

The initial \textsc{Evaluator} prompt is mostly a wakeup rule:

\begin{lstlisting}[style=prompt,label={lst:cloudcast-evaluator-init}]
Wake up whenever Implementer has written new code since the last eval.
If the task uses pure-Python code, you can eval directly without waiting
for a Builder green. Do not eval twice in a row without new code in
between, and do not eval on a build/import failure.
\end{lstlisting}

After early \(\Delta=0\) evaluations, the first mutation adds a pre-evaluation validation layer:

\begin{lstlisting}[style=prompt,label={lst:cloudcast-evaluator-gen1}]
Do: run quick local checks (syntax/import, lightweight tests, and a
schema/validator smoke test), use inexpensive heuristics to estimate
whether the change is likely to yield a meaningful improvement, and
only call request_eval when those checks and the estimate indicate a
real chance of improvement.
\end{lstlisting}

A later mutation makes this check more structural:

\begin{lstlisting}[style=prompt,label={lst:cloudcast-evaluator-gen2}]
Before calling request_eval():
- Do: run a quick compile/import check, run lightweight unit/smoke
  tests or static sanity checks, verify required interfaces/markers
  remain intact, and run a deterministic local heuristic/estimate
  that justifies an expected improvement; document the rationale for
  why an eval is likely to change the score.
\end{lstlisting}

This illustrates the general mechanism. The \textsc{Evaluator} does not learn a new optimization
algorithm. It learns to avoid wasting expensive evaluation calls. The same motif then appears in
other roles: \textsc{Builder} learns static checks before build commands, \textsc{Implementer} learns
intent-and-verify editing, and \textsc{Finalizer} learns to recheck invariants before committing. Thus,
prompt evolution changes the wakeup decisions, and the changed wakeup decisions reshape the
auction-selected topology.
\section{Why the Generalist Does Not Monopolize}
\label{sec:case-study-generalist-no-monopoly}

This appendix studies a potential failure mode of the economy: if one agent is given access to all
tools, will it monopolize the market and eliminate specialization? In Finance-Agent-Bench, we
introduce a \emph{full-tool generalist} agent with access to all four tools
(\texttt{edgar\_search}, \texttt{web\_search}, \texttt{parse\_html\_page},
\texttt{retrieve\_information}), alongside the ordinary specialized agents. At first glance, one might
expect such an agent to dominate the society. It has the broadest action space, the least tool-level
restriction, and in principle can subsume filing search, web search, document parsing, retrieval, and
answer submission. Yet empirically it does not monopolize control. Specialized agents remain
persistently active and frequently win auctions.

This failure to monopolize is informative: in our economy, being more general is not the same as
being more competitive. The generalist must spread its prompt budget across heterogeneous
requirements: tool formatting, decomposition, temporal coverage, accounting consistency, numerical
verification, cross-source reconciliation, and final answer submission. As a result, its evolved prompt
becomes broad and procedurally cautious rather than sharply discriminative. This is useful for
coverage, but economically costly. A generalist that tries to be prepared for everything is often
outbid by a specialist whose wakeup condition, search pattern, and evidence standard are tuned to
one narrower subproblem. The generalist does not fail because it is weak; it fails to monopolize
because it is too general. The society rewards agents that are locally precise, not globally omnibus.

\subsection{Initial Generalist Prompt}

The initial generalist prompt is intentionally broad. It enforces tool-call correctness and a basic
search-first workflow, but it does not prescribe domain-specific decomposition, numerical discipline,
or an evidence hierarchy beyond a generic sequence of actions.

\begin{lstlisting}[style=prompt,label={lst:generalist-init}]
TOOL-CALL FORMAT RULES (do not violate):
- edgar_search: pass CIK numbers as strings, e.g. ciks=['0000002488'].
- web_search: call web_search(search_query="..."), NOT web_search(query="...").
- parse_html_page: pass a valid URL string.
- retrieve_information: query stored documents by key.
- final_answer: call final_answer(answer="...") to submit.
Work methodically: search first, retrieve details, then answer. Do ONE
tool call per turn. When evidence is sufficient, submit immediately.
\end{lstlisting}

This initial version is a broad operator's manual. It ensures syntactic correctness and a sensible
high-level order of operations, but it does not yet tell the agent how to reason under financial query
structure. In particular, it does not force explicit decomposition of the question, distinguish
intermediate from final evidence, align fiscal periods, or guard against subtle financial errors such as
mixing segment-level values with consolidated totals.

\subsection{Evolved Generalist Prompt}

After evolution, the generalist prompt becomes substantially more elaborate. The added instructions
do not specialize it around one micro-skill. Instead, they accumulate global requirements for acting
as a competent all-purpose financial analyst.

\begin{lstlisting}[style=prompt,label={lst:generalist-evolved}]
TOOL-CALL FORMAT RULES (do not violate):
- edgar_search: pass CIK numbers as strings, e.g. ciks=['0000002488'].
- web_search: call web_search(search_query="..."), NOT web_search(query="...").
- parse_html_page: pass a valid URL string.
- retrieve_information: query stored documents by key.
- final_answer: call final_answer(answer="...") to submit.
Work methodically: Decompose the query into discrete requirements. For
multi-year trends, report data for every intermediate interval and
prioritize consolidated/total figures for the primary entity over
segmented or regional data. Trace every numerical value directly to raw
tool outputs; do not adopt figures from dialogue summaries without
verification. Ensure consistent accounting definitions across all periods.
Cross-reference sources to resolve discrepancies. Perform a sanity check
on the magnitude and trend before submitting. Do ONE tool call per turn.
Submit immediately once the full temporal and categorical scope is
satisfied.
\end{lstlisting}

Relative to generation 0, the mutation improves rigor but does not make the agent sharper about one
task family. It adds a broad bundle of global cautionary rules: decompose the query into discrete
requirements, cover every relevant time interval, prefer consolidated totals over segment or regional
figures, trace numbers back to raw tool outputs, enforce accounting consistency, cross-check source
discrepancies, and sanity-check trends and magnitudes.

These additions make the generalist safer, but they also explain why it does not monopolize. The
prompt becomes a general compliance layer rather than a highly tuned specialist heuristic. The
evolved generalist is taught to be careful everywhere, rather than exceptional in one recurring local
decision.

\subsection{Comparison to Specialized Agent Evolution}

The contrast is clearest when compared with the specialized \textsc{Edgar} and \textsc{Tavily}
agents. Their initial prompts are narrow: the \textsc{Edgar} agent is only asked to find relevant
filings, and the \textsc{Tavily} agent is only asked to perform one targeted web search. After
evolution, both remain narrow, but become much more exacting about the specific failure modes of
their own tool domains.

\paragraph{Edgar agent.}
The evolved \textsc{Edgar} prompt does not try to become a universal analyst. It sharpens around
filing-specific correctness: identifying the exact entity, filing type, and fiscal period; distinguishing
aggregate totals from plan- or segment-specific values; checking filing date and fiscal year; locating
future-dated projections inside the latest official filing; and matching the numerical answer to the
exact qualifiers of the query. This is a specialization trajectory: the prompt is refined to avoid a
small set of high-value, filing-specific mistakes.

\paragraph{Tavily agent.}
Likewise, the evolved \textsc{Tavily} prompt becomes more disciplined about source reliability and
arithmetic. It learns to identify the knowledge gap before search, trace numerical claims back to
primary evidence, avoid substituting non-matching periods or metrics, cross-reference sources, and
recompute arithmetic from raw values. Again, this is not a move toward generality. It is a move toward
narrower epistemic discipline within one tool domain.

By contrast, the full-tool generalist must absorb all of these concerns at once. Its prompt therefore
grows by accumulating heterogeneous global obligations rather than repeatedly repairing one local
competence. Economically, this matters because auctions reward agents that are locally high-value
under a particular state. A specialist can be the best agent precisely when its narrow prompt applies.
The generalist is broad, but its local edge is diluted.

\subsection{Generalist Failure as Evidence for Specialization}

This case study illustrates a general principle of the economy: prompt evolution is most effective
when repeated feedback can be compressed into a small number of role-specific decision rules. When
an agent has a narrow tool interface and a narrow evidentiary responsibility, the mutator can improve
it by writing concrete instructions that target the same class of mistakes repeatedly. The resulting
prompt becomes sharper, more falsifiable, and more economically competitive.

The generalist does not enjoy this advantage. Because it is responsible for too many heterogeneous
subtasks, each mutation tends to add another layer of generic caution rather than strengthening one
specific inference pattern. Its prompt becomes longer, broader, and more procedurally safe, but not
more economically dominant. Thus, the generalist does not monopolize: over-generality is not
specialization, and in this market, specialization is what wins control.

\end{document}